\documentclass{article}

\PassOptionsToPackage{numbers}{natbib}
\usepackage[preprint]{neurips_2026}
\makeatletter
\renewcommand{\@noticestring}{Code: \url{https://github.com/ETH-DISCO/n-vium}}
\renewcommand{\@notice}{%
  \enlargethispage{3\baselineskip}%
  \@float{noticebox}[b]%
    \footnotesize\@noticestring%
  \end@float%
}
\makeatother

\usepackage[utf8]{inputenc}
\usepackage[T1]{fontenc}
\usepackage{lmodern}
\usepackage{hyperref}
\usepackage{url}
\usepackage{booktabs}
\usepackage{amsfonts}
\usepackage{nicefrac}
\usepackage{microtype}
\usepackage{xcolor}

\usepackage{amsmath, amsthm}
\usepackage{bm}
\usepackage{graphicx}
\usepackage{multirow}
\usepackage{colortbl}
\usepackage{enumitem}
\usepackage{caption, subcaption}
\usepackage{placeins}
\usepackage{float}
\usepackage{wrapfig}
\usepackage{afterpage}
\usepackage{soul}
\usepackage{etoolbox}
\usepackage{algorithm, algpseudocode}
\usepackage{tabularx}
\newcolumntype{Y}{>{\centering\arraybackslash}X}
\newcolumntype{P}[1]{>{\centering\arraybackslash}p{#1}}

\usepackage{standalone}
\usepackage{tikz}
\usetikzlibrary{positioning, arrows.meta, fit, backgrounds, calc,
                 patterns, decorations.pathreplacing, decorations.markings, intersections}
\usepackage{pgfplots, pgfplotstable}
\pgfplotsset{compat=1.18}
\usepgfplotslibrary{fillbetween, groupplots}

\newtheorem{proposition}{Proposition}

\definecolor{riserOne}  {RGB}{80,170, 95}
\definecolor{riserTwo}  {RGB}{140, 60,200}
\definecolor{riserThree}{RGB}{210,155, 15}
\definecolor{riserFour} {RGB}{210, 85, 25}
\definecolor{savings}   {RGB}{135,190,230}

\definecolor{posGreenBg} {RGB}{208, 239, 213}
\definecolor{posGreenFg} {RGB}{ 34, 120,  51}
\definecolor{posRedBg}   {RGB}{252, 218, 218}
\definecolor{posRedFg}   {RGB}{180,  30,  30}
\definecolor{posGrayBg}  {RGB}{235, 235, 235}
\definecolor{posGrayFg}  {RGB}{120, 120, 120}
\definecolor{posYellowBg}{RGB}{255, 243, 200}
\definecolor{posYellowFg}{RGB}{160, 110,   0}
\definecolor{posOursBg}  {RGB}{214, 230, 248}

\newcommand{\colorintensity}{20}
\definecolor{exitEarly1base}{RGB}{0,160,60}
\definecolor{exitEarly2base}{RGB}{255,180,0}
\definecolor{exitEarly3base}{RGB}{30,120,200}
\definecolor{exitFinalbase} {RGB}{230,80,30}
\colorlet{exitEarly1}{exitEarly1base!\colorintensity!white}
\colorlet{exitEarly2}{exitEarly2base!\colorintensity!white}
\colorlet{exitEarly3}{exitEarly3base!\colorintensity!white}
\colorlet{exitFinal} {exitFinalbase!\colorintensity!white}

\newcommand{\tok}[2]{\colorbox{#1}{\strut #2}}
\newcommand{\nv}{\ensuremath{N_{V}}}
\title{N-vium: Mixture-of-Exits Transformer for Accelerated Exact Generation}

\author{%
  Aleksander Lorenc \quad
  Frédéric Berdoz \quad
  Joël Mathys \quad
  Roger Wattenhofer\\[.5em]
  ETH Zurich\\
  \texttt{\{alorenc, fberdoz, jmathys, wattenhofer\}@ethz.ch}
}
\begin{document}

\def\adaptBaselineCE{2.4535}
\def\adaptBaselinePPL{11.63}

\def\adaptNviumCE{2.4479}
\def\adaptNviumPPL{11.56}
\def\adaptNviumSpeedup{57.9}

\def\adaptPPLimprovement{0.55}

\maketitle

\begin{abstract}

Improving the inference efficiency of autoregressive transformers typically means reducing FLOPs per token, usually through approximations that degrade model quality. We introduce N-vium, a mixture-of-exits transformer that partially parallelizes computation across depth on standard hardware, increasing effective FLOPs per second rather than minimizing compute per token. N-vium attaches prediction heads at multiple depths and defines the next-token distribution as a learned mixture over these exits, with token-adaptive routing. This formulation strictly generalizes the standard transformer, which is recovered exactly when routing assigns zero mass to all intermediate heads. Sampling from the mixture is exact, and complete KV caches are recovered by deferring the upper-layer computation and batching it with later tokens. We pretrain N-vium at scales up to 1.5B parameters. Our largest model reaches $\adaptNviumSpeedup$\% wall-clock speedup over a parameter- and data-matched standard transformer at no perplexity cost.
\end{abstract}

\section{Introduction}

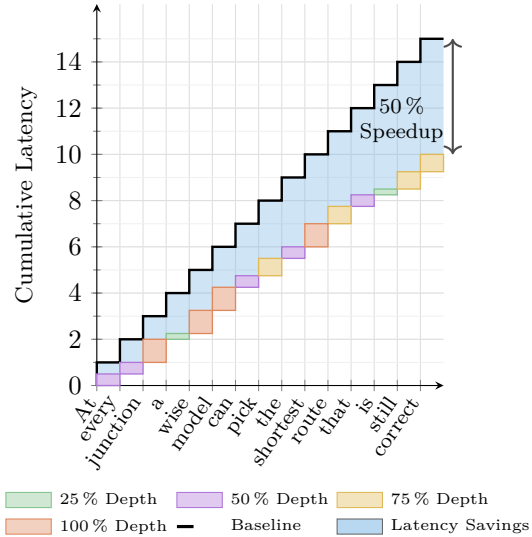
\begin{wrapfigure}[24]{r}{0.55\linewidth}
  \vspace{-1.9\baselineskip}
  \centering
  \newcommand{\lwBaseline}{0.9pt}
\newcommand{\lwStairs}  {0.7pt}
\newcommand{\lwRiser}   {1.2pt}
\def\speedupLabelY{11.5}

\begin{tikzpicture}
\begin{axis}[
  width  = \linewidth,
  unit vector ratio*={1 1 1},
  xmin = 0, xmax = 15,
  ymin = 0, ymax = 16.5,
  xtick = {0,1,...,14},
  xticklabels = {
    At, every, junction, a, wise, model, can, pick,
    the, shortest, route, that, is, still, correct
  },
  xticklabel style = {rotate=55, anchor=east, font=\footnotesize},
  ytick       = {0,2,...,14},
  minor ytick = {0,1,...,15},
  ylabel       = {Cumulative Latency},
  ylabel style = {font=\small},
  grid = both,
  major grid style = {gray!25, line width=0.4pt},
  minor grid style = {gray!12, line width=0.3pt},
  axis lines = left,
  clip       = false,
  legend columns = 3,
  legend style = {
    at={(0.5,-0.25)}, anchor=north,
    font=\tiny, cells={anchor=west},
    fill=white, fill opacity=0.9, draw=none, text opacity=1,
    column sep=1pt, inner sep=2pt,
  },
  legend image code/.code={\draw[#1,line width=\lwRiser] (0,0) -- (0.22cm,0);},
]

\addplot[
  name path = baseline, forget plot,
  black, line width = \lwBaseline,
] coordinates {
  (0,0)  (0,1)  (1,1)  (1,2)  (2,2)  (2,3)  (3,3)  (3,4)
  (4,4)  (4,5)  (5,5)  (5,6)  (6,6)  (6,7)  (7,7)  (7,8)
  (8,8)  (8,9)  (9,9)  (9,10) (10,10)(10,11)(11,11)(11,12)
  (12,12)(12,13)(13,13)(13,14)(14,14)(14,15)(15,15)
};

\addplot[
  name path = adaptive, forget plot,
  draw=none,
] coordinates {
  (0,0.00)  (0,0.50)  (1,0.50)
  (1,0.50)  (1,1.00)  (2,1.00)
  (2,1.00)  (2,2.00)  (3,2.00)
  (3,2.00)  (3,2.25)  (4,2.25)
  (4,2.25)  (4,3.25)  (5,3.25)
  (5,3.25)  (5,4.25)  (6,4.25)
  (6,4.25)  (6,4.75)  (7,4.75)
  (7,4.75)  (7,5.50)  (8,5.50)
  (8,5.50)  (8,6.00)  (9,6.00)
  (9,6.00)  (9,7.00)  (10,7.00)
  (10,7.00) (10,7.75) (11,7.75)
  (11,7.75) (11,8.25) (12,8.25)
  (12,8.25) (12,8.50) (13,8.50)
  (13,8.50) (13,9.25) (14,9.25)
  (14,9.25) (14,10.00)(15,10.00)
};

\addplot[savings, opacity=0.40, forget plot] fill between[of=adaptive and baseline];

\newcommand{\depthBlock}[5]{%
  \fill[#1!30] (axis cs:#2,#3) rectangle (axis cs:#4,#5);
  \draw[#1!70, line width=0.5pt] (axis cs:#2,#3) rectangle (axis cs:#4,#5);}

\depthBlock{riserOne}{3}{2.00}{4}{2.25}
\depthBlock{riserOne}{12}{8.25}{13}{8.50}

\depthBlock{riserTwo}{0}{0.00}{1}{0.50}
\depthBlock{riserTwo}{1}{0.50}{2}{1.00}
\depthBlock{riserTwo}{6}{4.25}{7}{4.75}
\depthBlock{riserTwo}{8}{5.50}{9}{6.00}
\depthBlock{riserTwo}{11}{7.75}{12}{8.25}

\depthBlock{riserThree}{7}{4.75}{8}{5.50}
\depthBlock{riserThree}{10}{7.00}{11}{7.75}
\depthBlock{riserThree}{13}{8.50}{14}{9.25}
\depthBlock{riserThree}{14}{9.25}{15}{10.00}

\depthBlock{riserFour}{2}{1.00}{3}{2.00}
\depthBlock{riserFour}{4}{2.25}{5}{3.25}
\depthBlock{riserFour}{5}{3.25}{6}{4.25}
\depthBlock{riserFour}{9}{6.00}{10}{7.00}

\draw[<->, thick, black!70] (axis cs:15.4, 15) -- (axis cs:15.4, 10);
\node[left, font=\footnotesize, align=center] at (axis cs:15.4, \speedupLabelY) {50\,\%\\Speedup};

\addlegendimage{area legend, fill=riserOne!30,   draw=riserOne!70}   \addlegendentry{25\,\% Depth}
\addlegendimage{area legend, fill=riserTwo!30,   draw=riserTwo!70}   \addlegendentry{50\,\% Depth}
\addlegendimage{area legend, fill=riserThree!30, draw=riserThree!70} \addlegendentry{75\,\% Depth}
\addlegendimage{area legend, fill=riserFour!30,  draw=riserFour!70}  \addlegendentry{100\,\% Depth}
\addlegendimage{black, line width=\lwBaseline}
  \addlegendentry{Baseline}
\addlegendimage{area legend, fill=savings, opacity=0.6}
  \addlegendentry{Latency Savings}

\end{axis}
\end{tikzpicture}
  \caption{Illustrative example of cumulative latency of Quadrivium vs.\ a standard transformer.
    By early sampling, Quadrivium achieves up to 57\% reduction in cumulative latency.
    Tokens are routed to shallower prediction heads, reducing the sequential depth of each decoding step.
    }
  \label{fig:latency-staircase}
\end{wrapfigure}

In a standard autoregressive large language model (LLM), tokens are processed sequentially through a stack of decoder layers. At every layer, the hidden representation of the token is used to compute a key-value (KV) pair that is cached for future tokens to attend to. After exiting the final layer, the hidden representation is projected onto the vocabulary to produce the next-token prediction. The forward pass therefore serves two distinct roles: building KV representations at every layer for the computation of future tokens, and processing the current token to predict the next one after the last layer. This procedure couples both computations through all layers and requires the full forward pass each time.

However, empirical investigations of latent representations using the logit
lens~\citep{nostalgebraist2020logitlens, belrose2023eliciting} show that for many tokens,
the predicted distribution converges well before the final
layer~\citep{schuster2021consistent, geva2022transformer}. Although full depth is
still needed to provide complete KV representations for future tokens to attend to, it suggests that not every token requires full processing depth for immediate prediction. Therefore, in principle, these two functions can be
decoupled: the current token processing could predict the next token early and complete its remaining future representations later.

Early exit methods exploit this observation by attaching prediction heads at intermediate
layers~\citep{xin2020deebert, elbayad2020depth,
schuster2021consistent, elhoushi2024layerskip} to reduce FLOPs per token in favor of decreasing latency. However, these methods train and treat intermediate layer predictions as approximations of the final layer ~$L$. This inherently approximate framing has several undesirable drawbacks. The decision when to exit requires additional hand-crafted confidence thresholds that need
calibration~\citep{xin2020deebert, schuster2021consistent}. Moreover, early predictions are weaker than full-depth ones due to a train-test mismatch, so speed comes at the cost of quality. Finally, when a token exits early, its KV caches for the bypassed layers are typically skipped to reduce FLOPs per token.

We propose a different approach. Instead of treating early predictions as approximations of
the final-layer distribution, we redefine the model itself as a \emph{mixture over exits}:
$\smash{\pi_{\mathrm{mix}} = \sum_{k=1}^{N} p_k\, \pi_k}$, where $\pi_k$ is the distribution at exit~$k$ and $p_k$ are learned token-adaptive mixture weights. This
formulation extends and strictly generalizes the standard transformer, which is recovered when all
early routing weights are zero. Because the model is a mixture, each prediction is exact and is sampled from the correct distribution optimized during training. It can naturally decouple the two roles of the full forward pass: next-token prediction happens at the selected exit, whereas KV representations are built across all layers independently on demand. Specifically, the remaining layer computations needed for KV caches can be deferred and batched
in parallel with future tokens, similarly to the parallel verification step in speculative decoding~\citep{leviathan2023fast, bae2023fast}. 
The speedup therefore comes from increasing FLOPs/second on standard hardware without any additional inference computation, not from reducing FLOPs per token. 

\begin{figure}[t]
  \centering
  \input{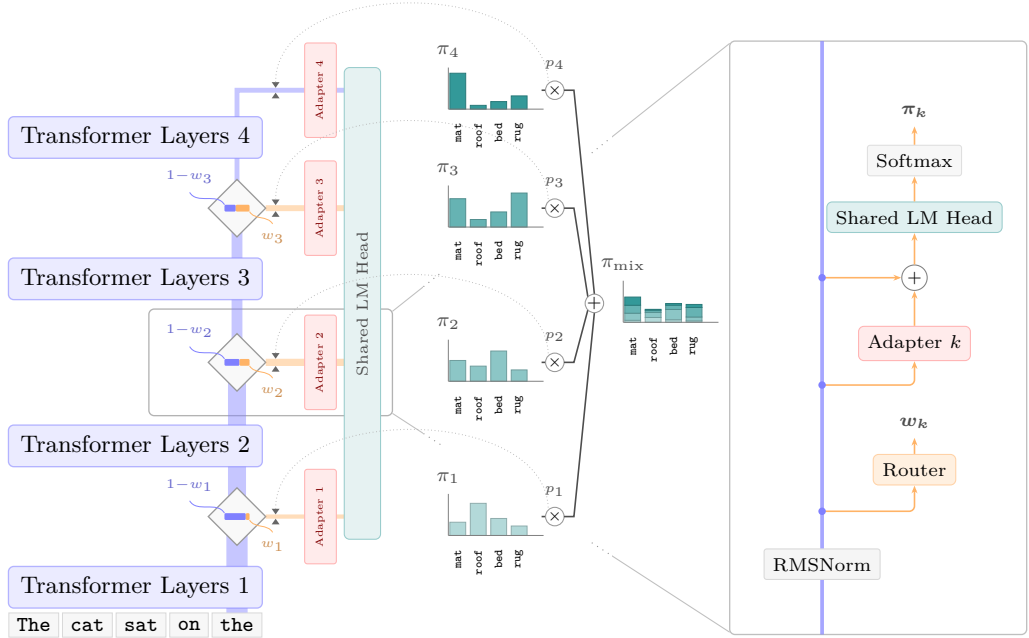}
  \caption{Overview of the N-vium architecture. The backbone (left) processes the input through $N$ blocks of transformer layers. At each junction, a router (diamond) outputs an exit probability $w_k$ that determines the likelihood of selecting the routed exit path (orange). Exit tokens pass through an adapter and the shared LM head to produce the exit distribution $\pi_k$. The mixture distribution $\pi_{\mathrm{mix}}$ is formed as a weighted combination of the four exit distributions with mixture weights $p_k$. The inset (right) zooms into a single junction.}
  \label{fig:overview}
\end{figure}
We introduce \textbf{N-vium}, a mixture-of-exits framework that is trained to directly optimize this mixture objective. The name derives from the Latin suffix \emph{-vium},
denoting a meeting of ways (e.g., \emph{bivium}, \emph{trivium}, \emph{quadrivium}),
reflecting the routing decision at each sampling junction. In this view, a standard transformer
 is a \emph{univium}: a single junction at the final layer.

We summarize our contributions as follows:
\begin{itemize}[leftmargin=*, topsep=2pt, itemsep=1pt]
    \item We introduce N-vium, a mixture-of-exits framework extension to standard decoder-only transformer architecture which enables learned, token-adaptive depth routing with exact early sampling from the token distribution. 

    \item We show how to train N-vium in practice, including how to steer models towards efficiency during training, and investigate its routing behavior.

    \item We pre-train and evaluate our proposed method across multiple scales up to 1.5B parameters and achieve consistent wall-clock speedup of up to $\adaptNviumSpeedup$\% compared to parameter- and data-matched transformer baseline at no perplexity degradation.

\end{itemize}

\begin{table}[!t]
  \caption{Comparison of conditional computation and efficient inference methods.
    \textbf{Exact inference}: the method samples from the trained distribution without approximation.
    \textbf{Learned routing}: per-token computation is allocated by a trained router.
    \textbf{No added FLOPs}: minimal per token inference compute overhead compared to the dense baseline
    \textbf{Decode-time speedup}: reduces wall-clock latency during autoregressive generation.
    }
  \label{tab:positioning}
  \newcommand{\yes}{\cellcolor{posGreenBg}\textcolor{posGreenFg}{\textbf{\checkmark}}}
  \newcommand{\no} {\cellcolor{posRedBg}\textcolor{posRedFg}{\textbf{\texttimes}}}
  \newcommand{\prt}{\cellcolor{posYellowBg}\textcolor{posYellowFg}{\textbf{\textasciitilde}}}
  \newcommand{\na} {\cellcolor{posGrayBg}\textcolor{posGrayFg}{---}}
  \centering
  \renewcommand{\arraystretch}{1}
  \small
  \begin{tabularx}{\textwidth}{P{3cm} YYYYYYY}
    \toprule
    \textbf{Criterion}
      & \textbf{Early exit}
      & \textbf{LayerSkip}
      & \textbf{Spec.\ dec.}
      & \textbf{MoE}
      & \textbf{MoD}
      & \textbf{Pruning}
      & \textbf{N-vium} \\
    \midrule
    Exact inference   & \no  & \yes & \yes & \yes & \no & \no  & \yes \\
    Learned routing   & \no  & \no  & \na  & \yes & \yes & \na  & \yes \\
    No added FLOPs    & \no & \no  & \no  & \yes & \yes & \yes & \yes \\
    Decode-time speedup
                      & \yes & \yes & \yes & \no & \yes  & \yes & \yes \\
    \bottomrule
  \end{tabularx}
\end{table}

\section{Related Work}

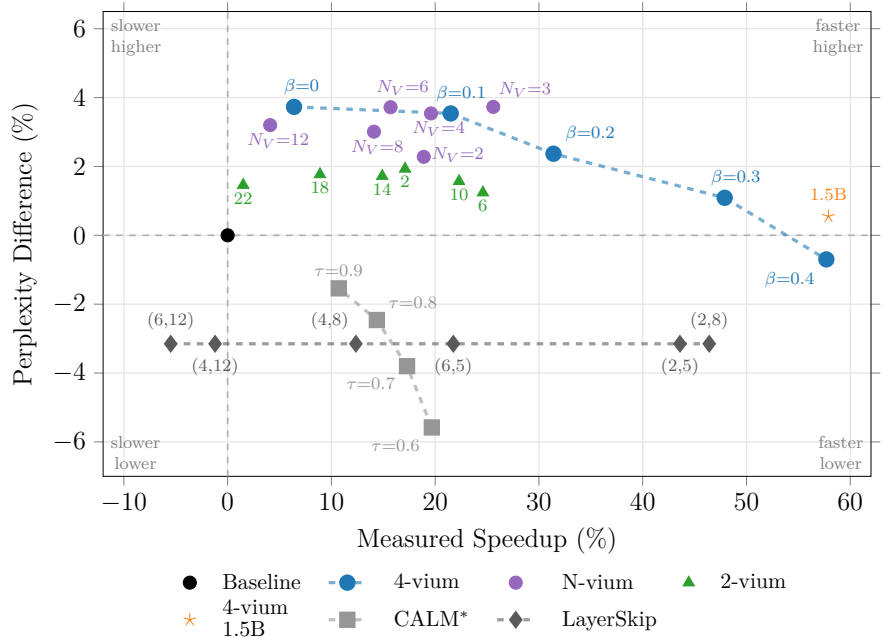
\begin{figure}[t!]
  \centering
  \def\paletteChoice{2}

\ifcase\paletteChoice
  \definecolor{colFourVium}     {RGB}{0,   114, 178}
  \definecolor{colFourViumLarge}{RGB}{213,  94,   0}
  \definecolor{colNVium}        {RGB}{ 86, 180, 233}
  \definecolor{colTwoVium}      {RGB}{  0, 158, 115}
\or
  \definecolor{colFourVium}     {RGB}{  0, 119, 187}
  \definecolor{colFourViumLarge}{RGB}{238, 119,  51}
  \definecolor{colNVium}        {RGB}{ 51, 187, 238}
  \definecolor{colTwoVium}      {RGB}{  0, 153, 136}
\or
  \definecolor{colFourVium}     {RGB}{ 31, 119, 180}
  \definecolor{colFourViumLarge}{RGB}{255, 127,  14}
  \definecolor{colNVium}        {RGB}{148, 103, 189}
  \definecolor{colTwoVium}      {RGB}{ 44, 160,  44}
\fi

\def\paretoScale{1.0}
\def\paretoMarkSize{3pt}
\def\paretoLineWidth{1.5pt}
\def\paretoLineOpacity{0.6}

\tikzset{
  lblNvBetaZero/.style  = {xshift=-8pt,  yshift=2pt},
  lblNvBetaOne/.style   = {xshift=-10pt,  yshift=2pt},
  lblNvBetaTwo/.style   = {xshift=2pt,  yshift=2pt},
  lblNvBetaThree/.style = {xshift=20pt, yshift=2pt},
  lblNvBetaFour/.style  = {xshift=-2pt, yshift=-2pt},
  lblCalmNineZero/.style = {xshift=0pt, yshift=2pt},
  lblCalmEight/.style    = {xshift=2pt,  yshift=2pt},
  lblCalmSeven/.style    = {xshift=-2pt, yshift=-2pt},
  lblCalmSix/.style      = {xshift=-2pt, yshift=-2pt},
  lblNvNTwo/.style    = {xshift=0pt,  yshift=-6pt},
  lblNvNThree/.style  = {xshift=-1pt,  yshift=1pt},
  lblNvNFour/.style   = {xshift=-12pt, yshift=-14pt},
  lblNvNSix/.style    = {xshift=-10pt, yshift=2pt},
  lblNvNEight/.style  = {xshift=-14pt, yshift=-14pt},
  lblNvNTwelve/.style = {xshift=-14pt,  yshift=-14pt},
  lblNvPosTwo/.style      = {xshift=0pt,  yshift=-12pt},
  lblNvPosFour/.style     = {xshift=2pt,   yshift=-12pt},
  lblNvPosSix/.style      = {xshift=0pt,   yshift=-12pt},
  lblNvPosEight/.style    = {xshift=2pt,   yshift=2pt},
  lblNvPosTen/.style      = {xshift=0pt, yshift=-12pt},
  lblNvPosTwelve/.style   = {xshift=2pt,   yshift=2pt},
  lblNvPosFourteen/.style   = {xshift=0pt, yshift=-12pt},
  lblNvPosSixteen/.style    = {xshift=-14pt, yshift=2pt},
  lblNvPosEighteen/.style   = {xshift=0pt,   yshift=-12pt},
  lblNvPosTwentyTwo/.style  = {xshift=0pt,   yshift=-12pt},
  lblNvOneFiveB/.style = {xshift=0pt, yshift=4pt},
  lblLsAA/.style = {xshift=0pt,  yshift=3pt},
  lblLsAB/.style = {xshift=0pt,  yshift=-3pt},
  lblLsBA/.style = {xshift=-12pt, yshift=3pt},
  lblLsBB/.style = {xshift=0pt,  yshift=-3pt},
  lblLsCA/.style = {xshift=0pt,  yshift=-3pt},
  lblLsCB/.style = {xshift=0pt,  yshift=3pt},
}

\scalebox{0.82}{%
\begin{tikzpicture}
\begin{axis}[
  width  = \linewidth,
  height = 0.65\linewidth,
  xlabel = {Measured Speedup (\%)},
  ylabel = {Perplexity Difference (\%)},
  xmin = -12, xmax = 62,
  ymin = -7,  ymax =  6.5,
  xtick = {-10, 0, 10, 20, 30, 40, 50, 60},
  ytick = {-6, -4, -2, 0, 2, 4, 6},
  grid  = both,
  grid style = {gray!20, line width=0.6pt},
  tick align = outside,
  font = \large,
  label style = {font=\large},
  tick label style = {font=\large},
  legend cell align = left,
  legend columns = 4,
  legend style = {
    at={(0.5,-0.18)}, anchor=north,
    font=\normalsize, draw=none, fill=none, inner sep=4pt,
    column sep=10pt,
  },
  clip = false,
]

\draw[black!30, dashed, line width=0.8pt]
  (axis cs:0,-7) -- (axis cs:0,6.5);
\draw[black!30, dashed, line width=0.8pt]
  (axis cs:-12,0) -- (axis cs:62,0);

\addplot[
  color = black,
  mark = *,
  mark size = \paretoMarkSize,
  only marks,
  mark options = {fill=black, solid, opacity=1},
] coordinates {(0, 0)};
\addlegendentry{Baseline}

\node[font=\footnotesize, gray, anchor=north east, align=center] at (axis cs: 62,  6.5) {faster\\higher};
\node[font=\footnotesize, gray, anchor=south east, align=center] at (axis cs: 62, -7)   {faster\\lower};
\node[font=\footnotesize, gray, anchor=north west, align=center] at (axis cs:-12,  6.5) {slower\\higher};
\node[font=\footnotesize, gray, anchor=south west, align=center] at (axis cs:-12, -7)   {slower\\lower};

\addplot[
  color      = {colFourVium},
  mark        = *,
  mark size   = \paretoMarkSize,
  dashed, line width=\paretoLineWidth, opacity=\paretoLineOpacity,
  mark options= {fill={colFourVium}, solid, opacity=1},
] coordinates {
  ( 6.4,  3.73)
  (21.5,  3.54)
  (31.4,  2.37)
  (47.9,  1.09)
  (57.7, -0.70)
};
\addlegendentry{4-vium}

\addplot[
  color      = {colNVium},
  mark        = *,
  mark size   = \paretoMarkSize,
  only marks,
  mark options= {fill={colNVium}, solid, opacity=1},
] coordinates {
  (18.9,  2.28)
  (25.6,  3.73)
  (19.6,  3.54)
  (15.7,  3.72)
  (14.1,  3.01)
  ( 4.1,  3.20)
};
\addlegendentry{N-vium}

\node[font=\footnotesize, anchor=south west, lblNvNTwo,
      color={colNVium}]
  at (axis cs:18.9, 2.28) {$N_V{=}2$};
\node[font=\footnotesize, anchor=south west, lblNvNThree,
      color={colNVium}]
  at (axis cs:25.6, 3.73) {$N_V{=}3$};
\node[font=\footnotesize, anchor=south west, lblNvNFour,
      color={colNVium}]
  at (axis cs:19.6, 3.54) {$N_V{=}4$};
\node[font=\footnotesize, anchor=south west, lblNvNSix,
      color={colNVium}]
  at (axis cs:15.7, 3.72) {$N_V{=}6$};
\node[font=\footnotesize, anchor=south west, lblNvNEight,
      color={colNVium}]
  at (axis cs:14.1, 3.01) {$N_V{=}8$};
\node[font=\footnotesize, anchor=south west, lblNvNTwelve,
      color={colNVium}]
  at (axis cs: 4.1, 3.20) {$N_V{=}12$};

\addplot[
  color      = {colTwoVium},
  mark        = triangle*,
  mark size   = \paretoMarkSize,
  only marks,
  mark options= {fill={colTwoVium}, solid, opacity=1},
] coordinates {
  (17.1,  1.93)
  (24.6,  1.24)
  (22.3,  1.57)
  (14.9,  1.72)
  ( 8.9,  1.77)
  ( 1.5,  1.46)
};
\addlegendentry{2-vium}

\node[font=\footnotesize, anchor=south, lblNvPosTwo,
      color={colTwoVium}]
  at (axis cs:17.1, 1.93) {2};
\node[font=\footnotesize, anchor=south, lblNvPosSix,
      color={colTwoVium}]
  at (axis cs:24.6, 1.24) {6};
\node[font=\footnotesize, anchor=south, lblNvPosTen,
      color={colTwoVium}]
  at (axis cs:22.3, 1.57) {10};
\node[font=\footnotesize, anchor=south, lblNvPosFourteen,
      color={colTwoVium}]
  at (axis cs:14.9, 1.72) {14};
\node[font=\footnotesize, anchor=south, lblNvPosEighteen,
      color={colTwoVium}]
  at (axis cs: 8.9, 1.77) {18};
\node[font=\footnotesize, anchor=south, lblNvPosTwentyTwo,
      color={colTwoVium}]
  at (axis cs: 1.5, 1.46) {22};

\addplot[
  color      = {colFourViumLarge},
  mark        = star,
  mark size   = \paretoMarkSize,
  only marks,
  mark options= {fill={colFourViumLarge}, solid, opacity=1},
] coordinates {
  (\adaptNviumSpeedup, \adaptPPLimprovement)
};
\addlegendentry{\shortstack[l]{4-vium\\1.5B}}

\node[font=\footnotesize, anchor=south, lblNvOneFiveB,
      color={colFourViumLarge}]
  at (axis cs:\adaptNviumSpeedup, \adaptPPLimprovement) {1.5B};

\node[font=\footnotesize, anchor=south west, lblNvBetaZero,
      color={colFourVium}]
  at (axis cs: 6.4,  3.73) {$\beta{=}0$};
\node[font=\footnotesize, anchor=south west, lblNvBetaOne,
      color={colFourVium}]
  at (axis cs:21.5,  3.54) {$\beta{=}0.1$};
\node[font=\footnotesize, anchor=south west, lblNvBetaTwo,
      color={colFourVium}]
  at (axis cs:31.4,  2.37) {$\beta{=}0.2$};
\node[font=\footnotesize, anchor=south east, lblNvBetaThree,
      color={colFourVium}]
  at (axis cs:47.9,  1.09) {$\beta{=}0.3$};
\node[font=\footnotesize, anchor=north east, lblNvBetaFour,
      color={colFourVium}]
  at (axis cs:57.7, -0.70) {$\beta{=}0.4$};

\addplot[
  color      = {rgb,255: red,150; green,150; blue,150},
  mark        = square*,
  mark size   = \paretoMarkSize,
  dashed, line width=\paretoLineWidth, opacity=\paretoLineOpacity,
  mark options= {fill={rgb,255: red,150; green,150; blue,150}, solid, opacity=1},
] coordinates {
  (19.68, -5.58)
  (17.29, -3.80)
  (14.38, -2.46)
  (10.71, -1.54)
};
\addlegendentry{CALM$^*$}

\node[font=\footnotesize, anchor=south, lblCalmNineZero,
      color={rgb,255: red,150; green,150; blue,150}]
  at (axis cs:10.71, -1.54) {$\tau{=}0.9$};
\node[font=\footnotesize, anchor=south west, lblCalmEight,
      color={rgb,255: red,150; green,150; blue,150}]
  at (axis cs:14.38, -2.46) {$\tau{=}0.8$};
\node[font=\footnotesize, anchor=north east, lblCalmSeven,
      color={rgb,255: red,150; green,150; blue,150}]
  at (axis cs:17.29, -3.80) {$\tau{=}0.7$};
\node[font=\footnotesize, anchor=north east, lblCalmSix,
      color={rgb,255: red,150; green,150; blue,150}]
  at (axis cs:19.68, -5.58) {$\tau{=}0.6$};

\addplot[
  color      = {rgb,255: red,90; green,90; blue,90},
  mark        = diamond*,
  mark size   = \paretoMarkSize,
  dashed, line width=\paretoLineWidth, opacity=\paretoLineOpacity,
  mark options= {fill={rgb,255: red,90; green,90; blue,90}, solid, opacity=1},
] coordinates {
  (-5.49, -3.15)
  (-1.21, -3.15)
  (12.36, -3.15)
  (21.75, -3.15)
  (43.59, -3.15)
  (46.41, -3.15)
};
\addlegendentry{LayerSkip}

\node[font=\footnotesize, anchor=south, lblLsAA, color={rgb,255: red,90; green,90; blue,90}]
  at (axis cs:-5.49, -3.15) {(6,12)};
\node[font=\footnotesize, anchor=north, lblLsAB, color={rgb,255: red,90; green,90; blue,90}]
  at (axis cs:-1.21, -3.15) {(4,12)};
\node[font=\footnotesize, anchor=south, lblLsBA, color={rgb,255: red,90; green,90; blue,90}]
  at (axis cs:12.36, -3.15) {(4,8)};
\node[font=\footnotesize, anchor=north, lblLsBB, color={rgb,255: red,90; green,90; blue,90}]
  at (axis cs:21.75, -3.15) {(6,5)};
\node[font=\footnotesize, anchor=north, lblLsCA, color={rgb,255: red,90; green,90; blue,90}]
  at (axis cs:43.59, -3.15) {(2,5)};
\node[font=\footnotesize, anchor=south, lblLsCB, color={rgb,255: red,90; green,90; blue,90}]
  at (axis cs:46.41, -3.15) {(2,8)};

\end{axis}
\end{tikzpicture}%
}
  \caption{Quality-speed Pareto front. Wall-clock speedup measured on an A100 GPU over a parameter-matched dense baseline against perplexity difference.
    All models use the same 24-layer, 768-wide backbone architecture.
    \textbf{4-vium}: labels show varying compute penalty $\beta$.
    \textbf{N-vium}: labels show varying number of exits $\nv$.
    \textbf{2-vium}: labels show the backbone attachment layer of the early sampling head.
    \textbf{4-vium 1.5B}: our largest and fastest model (1536-wide, 48-layer), trained with an adaptive $\beta$.
    \textbf{CALM$^*$}: our reimplementation of CALM~\citep{schuster2022confident}, labelled by exit threshold $\tau$.
    \textbf{LayerSkip}: labelled by (exit layer, draft length).}
  \label{fig:comparison-pareto}
\end{figure}

\paragraph{Inference efficiency.}
Improving the efficiency of large language models is an active and broad research area.
Techniques such as knowledge distillation~\citep{hinton2015distilling, jiao2020tinybert, wang2020minilm, sanh2019distilbert}, pruning~\citep{fan2020reducing, michel2019sixteen},
and quantization~\citep{shen2020qbert, sun2020mobilebert, bai2021binarybert} reduce memory and arithmetic overhead.
Architectural modifications such as group query attention~\citep{ainslie2023gqa} reduce the KV-cache bandwidth with minimal quality loss.
Mixture-of-Experts (MoE) architectures~\citep{fedus2022switch, du2022glam, jiang2024mixtral, zhou2022mixture} achieve superior quality per unit of active compute
by routing each token to a sparse subset of expert layers.
Speculative decoding~\citep{stern2018blockwise, leviathan2023fast, chen2023accelerating, zhang2024draft} is a notable technique that accelerates generation without any quality degradation, by drafting tokens with a cheap model and verifying them in a single parallel pass of the target model.

\paragraph{Adaptive computation.}
Another line of work pursues a fundamentally different efficiency strategy: allocate computation \emph{adaptively} based on input difficulty.
One family of methods assigns each token a variable number of processing steps via learned routing: either by deciding per-layer whether
a token should execute or skip that layer~\citep{graves2016adaptive, dehghani2018universal, raposo2024mixture, liu2021faster},
or by controlling how many recursive or recurrent passes a token undergoes~\citep{bae2025mixture, zhu2025scaling, alizadeh2024duollm}.
Instead, a second family stops the forward pass of a token entirely at an intermediate layer, emitting a prediction early and skipping
all remaining computations~\citep{elhoushi2024layerskip, kavehzadeh2024sorted}.

\paragraph{Early exit.}
Early exit is a specific form of adaptive computation in which a token or sample stops processing in an intermediate layer
and emits a prediction directly from that point~\citep{yin2022adavit, wang2018skipnet, liu2021faster}.
The idea was extensively explored in the context of encoder models such as BERT, where confidence or patience criteria determine
the exit layer per input~\citep{xin2020deebert,
xin2021berxit, zhou2020bert, hou2020dynabert}.
Early exit has also been adapted to decoder-only autoregressive models, where methods use confidence scores, hidden-state saturation,
or learned classifiers to make per-token exit decisions~\citep{schuster2021consistent, schuster2022confident, elbayad2020depth, zeng2024consistentee}.
CALM~\citep{schuster2022confident} is a prominent example, applying per-token confidence-based halting and allowing
different tokens within the same sequence to consume different amounts of compute.

\paragraph{Layer saturation.}
A key observation motivating the early exit approach is that a token's prediction often converges well before the final layer.
Projecting intermediate hidden states onto the vocabulary~\citep{nostalgebraist2020logitlens, belrose2023eliciting} reveals
that for many tokens, the output distribution stabilizes early~\citep{kaya2019shallow, shan2024early, geva2022transformer}
This motivates routing easy tokens to shallow prediction heads and reserving full depth for harder ones.

\paragraph{Positioning.}
Our work differs from each of these lines of research along several axes. Early exit methods such as CALM
retrofit variable-depth inference onto pretrained backbones via confidence thresholds, skip
upper-layer computation entirely, and approximate the resulting KV cache gaps, introducing train/test mismatch,
inexact sampling, and quality degradation. We instead train from scratch with early sampling as a first-class
objective, recover all KV caches exactly by deferring upper-layer computation and batching it with future
tokens, and sample exactly from the trained distribution. No computation is skipped. Speculative decoding
achieves lossless acceleration, but requires a separate draft model
and accept/reject sampling. Mixture of experts (MoE) architectures route in the \emph{width} dimension, selecting which parameters
to activate, and do not reduce per-token latency. Mixture of Depths (MoD) routes tokens to skip individual layers. It introduces approximations at inference time and wall-clock decode speedup is not guaranteed.

Our method is, to our knowledge, the first to satisfy all of these criteria:
exact sampling, full KV cache recovery, learned routing, and decode-time speedup.
Table~\ref{tab:positioning} summarizes this comparison between methods.

\section{Background}

We introduce N-vium, which extends the standard decoder transformer to a mixture-of-exits architecture. Notably, our technique itself is rather lightweight and, as such, compatible with most transformer variants, including different attention mechanisms, positional encodings, and activation functions.

In the following, we introduce the necessary naming and variable convention. Given a token sequence $x_1, \ldots, x_T$, each token $x_t$ is embedded into a $d$-dimensional vector $\mathbf{h}_t^{(0)} \in \mathbb{R}^d$, initializing the residual stream.
The model applies $L$ decoder transformer layers, each refining the stream and consisting of the common causal attention, MLP and normalization components.
Let $\mathcal{V}$ denote the vocabulary. The final hidden state $\mathbf{h}_t^{(L)} \in \mathbb{R}^d$ determines the logits $\mathbf{z}_t = W_{\mathrm{lm}}\, \mathbf{h}_t^{(L)}$, where $W_{\mathrm{lm}} \in \mathbb{R}^{|\mathcal{V}| \times d}$ denotes the language modeling head.
Softmax gives the normalized next token distribution $p(x_{t+1} \mid x_{1:t}) = \mathrm{softmax}(\mathbf{z}_t)$, and the model minimizes the expected negative log-likelihood of the training corpus using the next token prediction.
During autoregressive generation, each layer $\ell$ retains a \emph{layer state} $s_{t'}^{(\ell)}$ for every past token (the key-value projections of the attention sub-layer) also referred to as the KV cache.
Generating any future token requires $s_{t'}^{(\ell)}$ to be available for all $t'\leq t$ and all $\ell \in \{1,\ldots,L\}$.

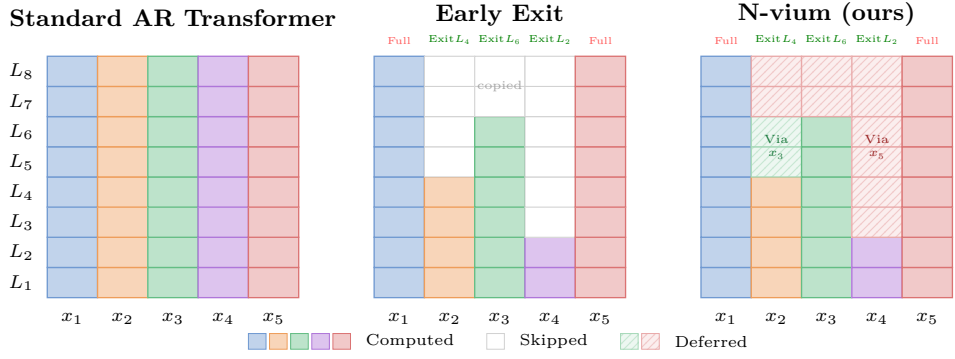
\begin{figure}[t]
  \centering
  \def\NCOLS{5}
\def\NROWS{8}
\def\NGAPFACTOR{1.5}
\pgfmathsetlengthmacro{\CWlen}{\textwidth / (3*\NCOLS + 2*\NGAPFACTOR + 3)}
\pgfmathsetmacro{\CW}{\CWlen / 1cm}
\pgfmathsetmacro{\CH}{\CW * 0.6}
\pgfmathsetmacro{\NGAP}{\CW * \NGAPFACTOR}

\definecolor{tone}  {RGB}{50,110,190}
\definecolor{ttwo}  {RGB}{230,115,20}
\definecolor{tthree}{RGB}{34,160,80}
\definecolor{tfour} {RGB}{140,60,200}
\definecolor{tfive} {RGB}{200,40,40}

\colorlet{skipfill}{white}
\colorlet{skipborder}{black!18}

\newcommand{\CellA}[2]{%
  \fill[tone!30]   (#1,#2) rectangle ++(1,1);
  \draw[tone!70,   line width=0.5pt] (#1,#2) rectangle ++(1,1);}

\newcommand{\CellB}[2]{%
  \fill[ttwo!30]   (#1,#2) rectangle ++(1,1);
  \draw[ttwo!70,   line width=0.5pt] (#1,#2) rectangle ++(1,1);}

\newcommand{\CellC}[2]{%
  \fill[tthree!30] (#1,#2) rectangle ++(1,1);
  \draw[tthree!70, line width=0.5pt] (#1,#2) rectangle ++(1,1);}

\newcommand{\CellD}[2]{%
  \fill[tfour!30]  (#1,#2) rectangle ++(1,1);
  \draw[tfour!70,  line width=0.5pt] (#1,#2) rectangle ++(1,1);}

\newcommand{\CellE}[2]{%
  \fill[tfive!25]  (#1,#2) rectangle ++(1,1);
  \draw[tfive!65,  line width=0.5pt] (#1,#2) rectangle ++(1,1);}

\newcommand{\Cskip}[2]{%
  \fill[skipfill] (#1,#2) rectangle ++(1,1);
  \draw[skipborder,line width=0.4pt] (#1,#2) rectangle ++(1,1);}

\newcommand{\CellPiggyC}[2]{%
  \fill[tthree!8]  (#1,#2) rectangle ++(1,1);
  \fill[pattern=north east lines, pattern color=tthree!25] (#1,#2) rectangle ++(1,1);
  \draw[tthree!40, line width=0.5pt] (#1,#2) rectangle ++(1,1);}

\newcommand{\CellPiggyD}[2]{%
  \fill[tfour!8]   (#1,#2) rectangle ++(1,1);
  \fill[pattern=north east lines, pattern color=tfour!25]  (#1,#2) rectangle ++(1,1);
  \draw[tfour!40,  line width=0.5pt] (#1,#2) rectangle ++(1,1);}

\newcommand{\CellPiggyE}[2]{%
  \fill[tfive!8]   (#1,#2) rectangle ++(1,1);
  \fill[pattern=north east lines, pattern color=tfive!25]  (#1,#2) rectangle ++(1,1);
  \draw[tfive!40,  line width=0.5pt] (#1,#2) rectangle ++(1,1);}

\begin{tikzpicture}[x=\CW cm, y=\CH cm]

  \pgfmathsetmacro{\PANELCM}{\NCOLS*\CW + \NGAP}
  \pgfmathsetmacro{\PANELGU}{\NCOLS + \NGAP/\CW}

  \foreach \r in {0,...,7} { \CellA{0}{\r} }
  \foreach \r in {0,...,7} { \CellB{1}{\r} }
  \foreach \r in {0,...,7} { \CellC{2}{\r} }
  \foreach \r in {0,...,7} { \CellD{3}{\r} }
  \foreach \r in {0,...,7} { \CellE{4}{\r} }

  \foreach \r/\l in {0/1, 1/2, 2/3, 3/4, 4/5, 5/6, 6/7, 7/8} {
    \node[font=\scriptsize, anchor=east] at (-0.08, \r+0.5) {$L_{\l}$};
  }

  \node[font=\scriptsize, anchor=north] at (0.5, -0.2) {$x_1$};
  \node[font=\scriptsize, anchor=north] at (1.5, -0.2) {$x_2$};
  \node[font=\scriptsize, anchor=north] at (2.5, -0.2) {$x_3$};
  \node[font=\scriptsize, anchor=north] at (3.5, -0.2) {$x_4$};
  \node[font=\scriptsize, anchor=north] at (4.5, -0.2) {$x_5$};

  \node[font=\small\bfseries, anchor=south] at (2.5, 8.7) {Standard AR Transformer};

  \begin{scope}[xshift=\PANELCM cm]

    \foreach \r in {0,...,7} { \CellA{0}{\r} }

    \foreach \r in {0,...,3} { \CellB{1}{\r} }
    \foreach \r in {4,...,7} { \Cskip{1}{\r} }

    \foreach \r in {0,...,5} { \CellC{2}{\r} }
    \foreach \r in {6,7}     { \Cskip{2}{\r} }

    \CellD{3}{0}  \CellD{3}{1}
    \foreach \r in {2,...,7} { \Cskip{3}{\r} }

    \foreach \r in {0,...,7} { \CellE{4}{\r} }

    \node[font=\scriptsize, anchor=north] at (0.5, -0.2) {$x_1$};
    \node[font=\scriptsize, anchor=north] at (1.5, -0.2) {$x_2$};
    \node[font=\scriptsize, anchor=north] at (2.5, -0.2) {$x_3$};
    \node[font=\scriptsize, anchor=north] at (3.5, -0.2) {$x_4$};
    \node[font=\scriptsize, anchor=north] at (4.5, -0.2) {$x_5$};

    \node[font=\fontsize{3.8pt}{4.5pt}\selectfont, color=red!65,         anchor=south] at (0.5, 8.1) {Full};
    \node[font=\fontsize{3.8pt}{4.5pt}\selectfont, color=green!50!black, anchor=south] at (1.5, 8.1) {Exit\hspace{0.3pt}$L_4$};
    \node[font=\fontsize{3.8pt}{4.5pt}\selectfont, color=green!50!black, anchor=south] at (2.5, 8.1) {Exit\hspace{0.3pt}$L_6$};
    \node[font=\fontsize{3.8pt}{4.5pt}\selectfont, color=green!50!black, anchor=south] at (3.5, 8.1) {Exit\hspace{0.3pt}$L_2$};
    \node[font=\fontsize{3.8pt}{4.5pt}\selectfont, color=red!65,         anchor=south] at (4.5, 8.1) {Full};

    \node[font=\small\bfseries, anchor=south] at (2.5, 8.7) {Early Exit};

    \node[font=\fontsize{4.5pt}{5pt}\selectfont, color=black!35]
      at (2.5, 7.0) {copied};

  \end{scope}

  \pgfmathsetmacro{\RPANELCM}{2*\PANELCM}

  \begin{scope}[xshift=\RPANELCM cm]

    \foreach \r in {0,...,7} { \CellA{0}{\r} }

    \foreach \r in {0,...,3} { \CellB{1}{\r} }
    \CellPiggyC{1}{4}  \CellPiggyC{1}{5}
    \CellPiggyE{1}{6}  \CellPiggyE{1}{7}

    \foreach \r in {0,...,5} { \CellC{2}{\r} }
    \CellPiggyE{2}{6}  \CellPiggyE{2}{7}

    \CellD{3}{0}  \CellD{3}{1}
    \foreach \r in {2,...,7} { \CellPiggyE{3}{\r} }

    \foreach \r in {0,...,7} { \CellE{4}{\r} }

    \node[font=\scriptsize, anchor=north] at (0.5, -0.2) {$x_1$};
    \node[font=\scriptsize, anchor=north] at (1.5, -0.2) {$x_2$};
    \node[font=\scriptsize, anchor=north] at (2.5, -0.2) {$x_3$};
    \node[font=\scriptsize, anchor=north] at (3.5, -0.2) {$x_4$};
    \node[font=\scriptsize, anchor=north] at (4.5, -0.2) {$x_5$};

    \node[font=\fontsize{3.8pt}{4.5pt}\selectfont, color=red!65,         anchor=south] at (0.5, 8.1) {Full};
    \node[font=\fontsize{3.8pt}{4.5pt}\selectfont, color=green!50!black, anchor=south] at (1.5, 8.1) {Exit\hspace{0.3pt}$L_4$};
    \node[font=\fontsize{3.8pt}{4.5pt}\selectfont, color=green!50!black, anchor=south] at (2.5, 8.1) {Exit\hspace{0.3pt}$L_6$};
    \node[font=\fontsize{3.8pt}{4.5pt}\selectfont, color=green!50!black, anchor=south] at (3.5, 8.1) {Exit\hspace{0.3pt}$L_2$};
    \node[font=\fontsize{3.8pt}{4.5pt}\selectfont, color=red!65,         anchor=south] at (4.5, 8.1) {Full};

    \node[font=\fontsize{4.5pt}{5pt}\selectfont, color=tthree!75!black,
          align=center] at (1.5, 5.0) {Via\\$x_3$};
    \node[font=\fontsize{4.5pt}{5pt}\selectfont, color=tfive!75!black,
          align=center] at (3.5, 5.0) {Via\\$x_5$};

    \node[font=\small\bfseries, anchor=south] at (2.5, 8.7) {N-vium (ours)};

  \end{scope}

  \pgfmathsetmacro{\FIGCTR}{(\PANELGU + \PANELGU + \NCOLS)/2}

  \def\LSQ{0.52}
  \def\LEGY{-1.45}
  \pgfmathsetmacro{\LSQhalf}{\LSQ/2}

  \pgfmathsetmacro{\LXCOMP}{\FIGCTR - 5.0}
  \pgfmathsetmacro{\LXSKIP}{\FIGCTR - 0.26}
  \pgfmathsetmacro{\LXPIGGY}{\FIGCTR + 2.4}

  \pgfmathsetmacro{\CsA}{\LXCOMP}
  \pgfmathsetmacro{\CsB}{\LXCOMP + 0.42}
  \pgfmathsetmacro{\CsC}{\LXCOMP + 0.84}
  \pgfmathsetmacro{\CsD}{\LXCOMP + 1.26}
  \pgfmathsetmacro{\CsE}{\LXCOMP + 1.68}

  \fill[tone!30]   (\CsA,\LEGY-\LSQhalf) rectangle ++(0.35,\LSQ);
  \draw[tone!70,   line width=0.3pt] (\CsA,\LEGY-\LSQhalf) rectangle ++(0.35,\LSQ);

  \fill[ttwo!30]   (\CsB,\LEGY-\LSQhalf) rectangle ++(0.35,\LSQ);
  \draw[ttwo!70,   line width=0.3pt] (\CsB,\LEGY-\LSQhalf) rectangle ++(0.35,\LSQ);

  \fill[tthree!30] (\CsC,\LEGY-\LSQhalf) rectangle ++(0.35,\LSQ);
  \draw[tthree!70, line width=0.3pt] (\CsC,\LEGY-\LSQhalf) rectangle ++(0.35,\LSQ);

  \fill[tfour!30]  (\CsD,\LEGY-\LSQhalf) rectangle ++(0.35,\LSQ);
  \draw[tfour!70,  line width=0.3pt] (\CsD,\LEGY-\LSQhalf) rectangle ++(0.35,\LSQ);

  \fill[tfive!25]  (\CsE,\LEGY-\LSQhalf) rectangle ++(0.35,\LSQ);
  \draw[tfive!65,  line width=0.3pt] (\CsE,\LEGY-\LSQhalf) rectangle ++(0.35,\LSQ);

  \pgfmathsetmacro{\LXcompTXT}{\LXCOMP + 2.15}
  \node[font=\tiny, anchor=west] at (\LXcompTXT, \LEGY) {Computed};

  \fill[skipfill]  (\LXSKIP,\LEGY-\LSQhalf) rectangle ++(0.35,\LSQ);
  \draw[skipborder,line width=0.3pt] (\LXSKIP,\LEGY-\LSQhalf) rectangle ++(0.35,\LSQ);
  \pgfmathsetmacro{\LXskipTXT}{\LXSKIP + 0.47}
  \node[font=\tiny, anchor=west] at (\LXskipTXT, \LEGY) {Skipped};

  \pgfmathsetmacro{\PsA}{\LXPIGGY}
  \pgfmathsetmacro{\PsB}{\LXPIGGY + 0.42}

  \fill[tthree!8]  (\PsA,\LEGY-\LSQhalf) rectangle ++(0.35,\LSQ);
  \fill[pattern=north east lines, pattern color=tthree!25] (\PsA,\LEGY-\LSQhalf) rectangle ++(0.35,\LSQ);
  \draw[tthree!40, line width=0.3pt] (\PsA,\LEGY-\LSQhalf) rectangle ++(0.35,\LSQ);

  \fill[tfive!8]   (\PsB,\LEGY-\LSQhalf) rectangle ++(0.35,\LSQ);
  \fill[pattern=north east lines, pattern color=tfive!25]  (\PsB,\LEGY-\LSQhalf) rectangle ++(0.35,\LSQ);
  \draw[tfive!40,  line width=0.3pt] (\PsB,\LEGY-\LSQhalf) rectangle ++(0.35,\LSQ);

  \pgfmathsetmacro{\LXpiggyTXT}{\LXPIGGY + 0.89}
  \node[font=\tiny, anchor=west] at (\LXpiggyTXT, \LEGY) {Deferred};

\end{tikzpicture}
  \caption{Comparison of inference strategies over 5 tokens with an 8-layer network.
    \textbf{Standard}: every token traverses all 40 layers sequentially.
    \textbf{Early exit} (e.g., CALM~\citep{schuster2022confident}): tokens exit at different depths reducing latency to 28 layer units. Missing KV entries are copied from the last available state. The sampled tokens and cache are approximate.
    \textbf{N-vium} (ours): tokens learn to exit, also reducing latency but are sampled from the exact $\pi_{\mathrm{mix}}$ distribution. Deferred layers are computed on demand with future tokens in parallel. We denote this as piggybacking as the KV cache is complete, exact and consistent with the sampled predictions.}
  \label{fig:piggybacking}
\end{figure}

\section{Method}

N-vium augments a standard decoder backbone of $L$ layers with $\nv$ exits.
The backbone is partitioned into equal $\nv$ blocks $b_1, \ldots, b_{\nv}$, with junction layers placed evenly at $\ell_k = k \cdot L/\nv$ for $k = 1, \ldots, \nv$, so that block $b_k$ spans layers $\ell_{k-1}{+}1$ through $\ell_k$ (where $\ell_0 = 0$).
Each junction branches off the residual stream at $\mathbf{h}_t^{(\ell_k)}$ and gives the model the option to produce a next-token distribution $\pi_{k}$ without completing the remaining layers. Each junction operates on the hidden state $\mathbf{h}_t^{(\ell_k)}$ and consists of three components.
A \emph{router} MLP accepts $\mathbf{h}_t^{(\ell_k)}$ and produces $w_k(h_t^{(\ell_k)}) \in [0,1]$, the conditional probability of token $x_t$ sampling at junction $k$ given the current hidden state. By convention $w_{\nv} = 1$. The joint mixture-of-exit token distribution $\pi_{\mathrm{mix}}$ is a weighted sum of all $\pi_k$ according to the sampling probability $p_k$.

\begin{equation*}
  p_k(x_t) = w_k(h_t^{(\ell_k)}) \prod_{j=1}^{k-1}(1 - w_j(h_t^{(\ell_j)}))    \qquad
  \pi_{\mathrm{mix}}(x_{t+1} \mid x_{1:t}) = \sum_{k=1}^{\nv} p_k(x_t)\, \pi_k(x_{t+1} \mid x_{1:t})
\end{equation*}

If a token decides to sample at the current junction, an \emph{adapter} MLP transforms $\mathbf{h}_t^{(\ell_k)}$ into a representation $o_k$ aligned with the geometry expected by the shared projection matrix $W_{\mathrm{lm}}$.
The junction distribution $\pi_k(\,\cdot \mid x_{1:t})$ is then calculated as $\text{softmax}(o_k W_{\mathrm{lm}})$ of the adapted state using the language modeling head. Note that all sampling junctions share $W_{\mathrm{lm}}$ with the final head. The lightweight adapters allow each junction to specialize without a separate full-size projection.
A complete overview of the architecture is shown in Figure~\ref{fig:overview}.

The number of extra parameters introduced by the routers and adapters depends on the depth and width of the model and becomes negligible as the model scales. Across all models considered in this work, the average additional parameter overhead is 3.3\%, making it a lightweight addition. All dense baselines throughout this paper are parameter-matched to the corresponding N-vium model.

\subsection{Training}
\label{sec:training}

\begin{wrapfigure}[14]{r}{0.48\linewidth}
  \vspace{-2.2\baselineskip}
  \begin{minipage}{\linewidth}
  \begin{algorithm}[H]
  \caption{N-vium inference sampling.}
  \label{alg:decoding}
  \begin{algorithmic}[1]
  \Require junctions at layers $\ell_k = k \cdot \frac{L}{\nv}$\\
  $k \in \{1, \ldots, \nv\}$, $\ell_0 = 0$, $w_{\nv} = 1$
  \For{each generation step $t$}
    \For{$k = 1, \ldots, \nv$}
      \State Run Layers $\ell_{k-1}{+}1, \ldots, \ell_k$
      \State $b_t^{(k)} \sim \mathrm{Bernoulli}(w_k(h_t^{(\ell_k)}))$
      \If{$b_t^{(k)} = 1$} \Comment{junction $k$}
        \State $x_{t+1} \sim \pi_k(\,\cdot \mid x_{1:t})$
        \State \textbf{break}
      \EndIf
    \EndFor
  \EndFor
  \end{algorithmic}
  \end{algorithm}
  \end{minipage}
  \vspace{-4pt}
\end{wrapfigure}

The main training objective for N-vium-based transformer architectures also follows the standard transformer setup of predicting the next token and minimizes the cross-entropy loss. However, rather than a single distribution induced by the last layer state, we optimize the joint $\pi_{\mathrm{mix}}$ distribution. Note that this joint distribution is a weighted combination of intermediate predictions $\pi_k$. As such, we directly optimize the mixture  distribution from which we also sample at inference time. Compare this to the training objective in the early exit literature \citep{schuster2022confident}, which instead optimizes each intermediate distribution in isolation to match the underlying data.

 We introduce an auxiliary compute penalty that optimizes the expected normalized layer depth per token. This is intended to give fine-grained control and steer the model towards more efficient modes by shifting probability mass towards earlier predictions. 

\begin{equation*}
  \mathcal{L}_{\mathrm{mix}} = -\mathbb{E}_t\!\left[\log \pi_{\mathrm{mix}}(x_{t+1} \mid x_{1:t})\right] \quad \mathcal{L}_{\mathrm{compute}} = \mathbb{E}_t\!\left[\sum_{k=1}^{\nv} p_k(x_t)\,\frac{\ell_k}{L} \right]
  \label{eq:lmix}
\end{equation*}
The total loss objective consists of $\mathcal{L}_{\mathrm{main}} = \mathcal{L}_{\mathrm{mix}} + \beta\,\mathcal{L}_{\mathrm{compute}}$. Note that at $\beta=0$, we recover the pure next token optimization. To stabilize the routing distributions at initial training, we adapt the warm-up phase as detailed in the Appendix~\ref{app:impl}.
 Intuitively, the routers have to learn to assign appropriate weights to the intermediate predictions, so that the joint distribution is optimized for the next token under the desired model efficiency. Note that the architecture and training modifications are minimal to compute the $\pi_{\mathrm{mix}}$ distribution and therefore remain fully compatible with the standard training pipelines to properly utilize parallelized training. We continue the theoretical analysis in the Appendix~\ref{app:theory}.

\subsection{Inference}
\label{sec:inference}

To sample from the model at inference, we can either explicitly construct the $\pi_{\mathrm{mix}}$ distribution or make use of a mathematical equivalent formulation using a two step sampling approach to sample from the correct $\pi_{\mathrm{mix}}$ at intermediate predictions. At each junction $k$ we sample from $\mathrm{Bernoulli}(w_k(\mathbf{h}_t^{(\ell_k)}))$ determined by the routing probability. We then stop the current computation and draw a sample $x_{t+1}$ from $\pi_k(\cdot \mid x_{1:t})$ or continue the computation. Due to the two-step sampling according to $w_k$, the sample $x_{t+1}$ is actually sampled from the $\pi_{\mathrm{mix}}$ distribution for which we provide further details in Appendix~\ref{app:proof-multi}. 
The Algorithm~\ref{alg:decoding} outlines the full sampling procedure. 

Recall that the role of the full forward pass is two-fold. We want to draw the sample for the next token prediction while also building up proper KV representations that future tokens can attend to. To do this efficiently, we defer the computation of later KV representations once we decide to exit using an intermediate token prediction junction $k$. Instead, we rely on a technique, we name \textit{piggybacking}, to fill in the missing layers on demand when computing the forward pass for future tokens.  

Piggybacking is inspired by the parallel verification step in speculative decoding~\citep{stern2018blockwise, leviathan2023fast}.
Rather than computing missing states immediately, deferred tokens are grouped by their exit depth. As a future token's forward pass reaches each successive depth, all deferred tokens that terminated at that depth are \emph{piggybacked}: their hidden states are concatenated with the current token's hidden state and processed through the remaining layers in a parallel pass.
This fills in the missing layer states $s_{t'}^{(\ell)}$ for all deferred tokens $t' < t$ and all $\ell > \ell^*_{t'}$, at no additional sequential cost.
A similar technique SPD is used by~\citet{bae2023fast} to recover KV caches for a two-exit architecture. However, piggybacking generalizes naturally to any number of $\nv$ exit points, and the later computed hidden KV representations remain fully compatible with the sampled token. In SPD the final token representations, which determine KV, can induce a different next token distribution due to the approximate nature of the early sampling process.
Figure~\ref{fig:piggybacking} illustrates the difference between standard forward pass, early exit, and piggybacking.

As a result, the total backbone \emph{computation is identical} to that of a dense model. Each token eventually passes through all $L$ layers.
The speedup comes from reordering sequential computation into parallel batched passes in a learned fashion, exploiting GPU parallelism in the same way as the prefill phase \textit{without any  additional} or potential wasteful additional computation.

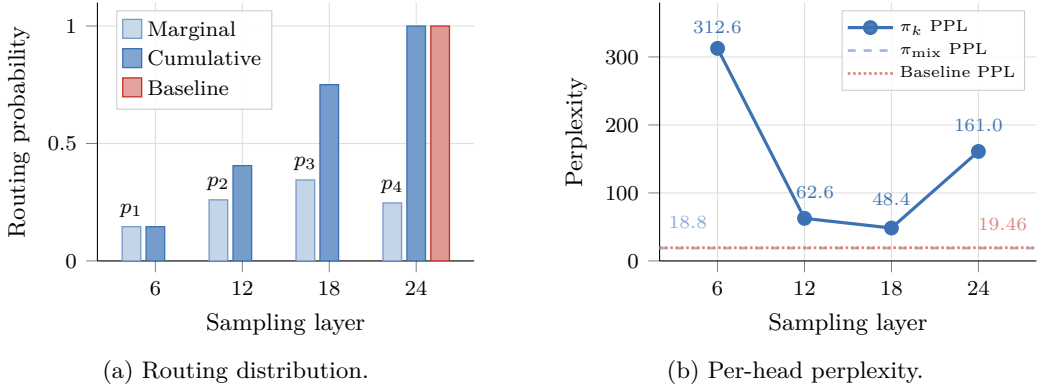
\begin{figure}[t!]
  \centering
  \definecolor{plotBlue}{RGB}{60,114,180}
\definecolor{plotRed} {RGB}{196, 52, 40}

\def\pLabelXShift{-9pt}
\def\lblSixAnchor   {north}
\def\lblSixXShift   {0pt}
\def\lblSixYShift   {14pt}

\def\lblTwelveAnchor{south}
\def\lblTwelveXShift{4pt}
\def\lblTwelveYShift{4pt}

\def\lblEighteenAnchor{south}
\def\lblEighteenXShift{0pt}
\def\lblEighteenYShift{5pt}

\def\lblTwentyfourAnchor{south}
\def\lblTwentyfourXShift{0pt}
\def\lblTwentyfourYShift{4pt}

\def\lblMixAnchor   {south west}
\def\lblMixXShift   {0pt}
\def\lblMixYShift   {4pt}

  \begin{subfigure}[t]{0.47\textwidth}
    \centering
    \begin{tikzpicture}
      \begin{axis}[
        width            = \linewidth,
        height           = 5.0cm,
        ybar,
        bar width        = 7pt,
        xmin             = 2,    xmax = 28,
        ymin             = 0,    ymax = 1.10,
        xtick            = {6, 12, 18, 24},
        xlabel           = {Sampling layer},
        xlabel style     = {font=\small},
        ylabel           = {Routing probability},
        ylabel style     = {font=\small},
        yticklabel style = {font=\footnotesize},
        xticklabel style = {font=\footnotesize},
        axis lines*      = left,
        grid             = both,
        major grid style = {gray!25, line width=0.4pt},
        minor grid style = {gray!12, line width=0.3pt},
        tick align       = outside,
        legend style     = {
          at             = {(0.05, 0.97)},
          anchor         = north west,
          font           = \footnotesize,
          cells          = {anchor=west},
          draw           = gray!40,
          fill           = white,
          fill opacity   = 0.9,
          text opacity   = 1,
          inner sep      = 2pt,
        },
        legend image code/.code={%
          \draw[#1] (0cm,-0.1cm) rectangle (0.25cm,0.15cm);
        },
      ]
        \addplot [fill=plotBlue!30, draw=plotBlue!70, line width=0.5pt] coordinates {
          (6,  0.145703)
          (12, 0.259934)
          (18, 0.344640)
          (24, 0.246770)
        };
        \addlegendentry{Marginal}

        \addplot [fill=plotBlue!70, draw=plotBlue, line width=0.5pt] coordinates {
          (6,  0.145703)
          (12, 0.405637)
          (18, 0.750277)
          (24, 1.000000)
        };
        \addlegendentry{Cumulative}

        \addplot [fill=plotRed!50, draw=plotRed, line width=0.5pt] coordinates {
          (24, 1.000000)
        };
        \addlegendentry{Baseline}

        \node [font=\footnotesize, anchor=south, xshift=\pLabelXShift] at (axis cs:6,  0.145703) {$p_1$};
        \node [font=\footnotesize, anchor=south, xshift=\pLabelXShift] at (axis cs:12, 0.259934) {$p_2$};
        \node [font=\footnotesize, anchor=south, xshift=\pLabelXShift] at (axis cs:18, 0.344640) {$p_3$};
        \node [font=\footnotesize, anchor=south, xshift=\pLabelXShift] at (axis cs:24, 0.246770) {$p_4$};
      \end{axis}
    \end{tikzpicture}
    \caption{Routing distribution.}
  \end{subfigure}
  \hfill
  \begin{subfigure}[t]{0.47\textwidth}
    \centering
    \begin{tikzpicture}
      \begin{axis}[
        width            = \linewidth,
        height           = 5.0cm,
        xmin             = 2,   xmax = 28,
        ymin             = 0,   ymax = 380,
        xtick            = {6, 12, 18, 24},
        xlabel           = {Sampling layer},
        xlabel style     = {font=\small},
        ylabel           = {Perplexity},
        ylabel style     = {font=\small},
        yticklabel style = {font=\footnotesize},
        xticklabel style = {font=\footnotesize},
        axis lines*      = left,
        grid             = both,
        major grid style = {gray!25, line width=0.4pt},
        minor grid style = {gray!12, line width=0.3pt},
        tick align       = outside,
        legend style     = {
          at             = {(0.97, 0.97)},
          anchor         = north east,
          legend columns = 1,
          font           = \tiny,
          cells          = {anchor=west},
          draw           = gray!40,
          fill           = white,
          fill opacity   = 0.9,
          text opacity   = 1,
          inner sep      = 1.5pt,
          row sep        = -1pt,
        },
      ]
        \addplot [plotBlue, thick, mark=*, mark size=2.2pt, line width=1.2pt] coordinates {
          (6,  312.6)
          (12,  62.6)
          (18,  48.4)
          (24, 161.0)
        };
        \addlegendentry{$\pi_k$ PPL}

        \addplot [plotBlue!50, dashed, line width=1.0pt] coordinates {
          (2, 18.8)(28, 18.8)
        };
        \addlegendentry{$\pi_{\mathrm{mix}}$ PPL}

        \addplot [plotRed!60, densely dotted, line width=1.0pt] coordinates {
          (2, 19.46)(28, 19.46)
        };
        \addlegendentry{Baseline PPL}

        \node [font=\scriptsize, color=plotRed!60,
               anchor=south east,
               xshift=0pt, yshift=3pt]
          at (axis cs:28, 19.46) {$19.46$};

        \node [font=\scriptsize, color=plotBlue,
               anchor=\lblSixAnchor,
               xshift=\lblSixXShift, yshift=\lblSixYShift]
          at (axis cs:6, 312.6) {$312.6$};

        \node [font=\scriptsize, color=plotBlue,
               anchor=\lblTwelveAnchor,
               xshift=\lblTwelveXShift, yshift=\lblTwelveYShift]
          at (axis cs:12, 62.6) {$62.6$};

        \node [font=\scriptsize, color=plotBlue,
               anchor=\lblEighteenAnchor,
               xshift=\lblEighteenXShift, yshift=\lblEighteenYShift]
          at (axis cs:18, 48.4) {$48.4$};

        \node [font=\scriptsize, color=plotBlue,
               anchor=\lblTwentyfourAnchor,
               xshift=\lblTwentyfourXShift, yshift=\lblTwentyfourYShift]
          at (axis cs:24, 161.0) {$161.0$};

        \node [font=\scriptsize, color=plotBlue!60,
               anchor=\lblMixAnchor,
               xshift=\lblMixXShift, yshift=\lblMixYShift]
          at (axis cs:2, 18.8) {$18.8$};
      \end{axis}
    \end{tikzpicture}
    \caption{Per-head perplexity.}
  \end{subfigure}
  \caption{Routing distribution and per-head perplexity for the trained
    4-vium model.
    \textbf{\\(a)} Marginal sampling probabilities $p_1, p_2, p_3, p_4$ and their cumulative sums at each junction.
    \textbf{(b)} The dashed line marks the mixture perplexity,
    which lies well below every individual head, confirming that routing
    learns to create a proper mixed distribution, rather than individual sampling heads.}
  \label{fig:exit-analysis}
\end{figure}

\section{Experiments}

We evaluate N-vium across multiple axes: We analyze the mechanisms of the trained model, provide a depth-width scaling study, a downstream benchmark transfer study via supervised fine-tuning, and a comparison against existing methods. Unless stated otherwise, all pretraining experiments use a hidden dimension of $768$ using $24$ layers and report perplexity differences relative to a parameter and data-matched baseline trained under identical conditions.

\paragraph{Datasets and baselines.}
All models are pretrained on the C4 dataset~\citep{raffel2020exploring} using the
LLaMA~2 architecture and tokenizer~\citep{touvron2023llama2}.
Each training uses $D=20N$ tokens for a given parameter count $N$,
following the regime of \citet{hoffmann2022training}.
The matched baseline in every experiment is a standard dense decoder-only transformer with the
same depth, slightly increased MLP width to match the additional parameters of N-vium and are trained with the cross-entropy loss on
the final head only.

\paragraph{Metrics and Speedup measurement.}
We report \emph{a relative difference in perplexity} (\%) on the matched baseline as the
primary pretraining metric, evaluated on a held-out C4 validation set of $5$M tokens:
positive values indicate that the mixture model achieves lower perplexity than the baseline counterpart. 
In addition, we measured wall-time inference speedup on a single A100 GPU using SpecBench~\citep{xia2024specbench}, a benchmark covering six prompt categories, averaged over three runs. We report a relative improvement in \emph{latency per token} (ms/token). Full implementation details, hyperparameters, and speedup measurement details are provided in Appendix~\ref{app:impl}.

\paragraph{Routing analysis.}
First, we validate whether the trained N-vium transformer exhibits the expected mixture-of-exits learned routing behavior. In Figure \ref{fig:exit-analysis} we can see that the Quadrivium model on 24 layers learns to allocate a significant portion of the probability mass towards earlier exits. Moreover, we verify that each exit contributes distinctively to the joint mixture rather than collapsing to the same approximate distribution. To this end, we evaluate the perplexity of each individual exit $\pi_k$ relative to the joint mixture $\pi_{\mathrm{mix}}$. Individual exits exhibit substantially higher perplexity than the joint mixture, confirming that exits learn complementary roles and that the mixture-of-exits training objective is effective.

\label{sec:exp-depth}

\paragraph{Depth-width scaling.}
We train Quadrivium models across two orthogonal axes of scale.
For depth scaling, we fix the hidden dimension at $768$ and vary the number of layers from $6$ to $48$.
For width scaling, we fix the depth at $24$ layers and vary the hidden dimension from $128$ to $1536$.
As Table~\ref{tab:depth-scaling} shows, we achieve consistent speedups across the entire model scale at no significant degradation in perplexity quality.

\pgfplotstableread{
layers  params  baselinece  modelce  speedupth  speedup
6       99      3.332       3.319    33.3        -8.3
12      141     3.135       3.104    41.4        10.9
24      226     2.939       2.902    48.0        26.8
36      311     2.826       2.788    53.6        36.4
48      396     2.769       2.720    58.2        41.6
}\depthdata

\pgfplotstableread{
width  params  baselinece  modelce   speedupth  speedup
128    13      4.357       4.188     24.3        53.8
256    36      3.579       3.524     62.3        35.3
512    112     3.103       3.058     45.3        25.0
768    226     2.939       2.902     48.5        26.8
1024   381     2.862       2.840     43.4        23.8
1536   806     2.59119     2.59073   56.7        32.0
}\widthdata

\begin{table}[t]
  \caption{Depth and width scaling results with $\beta = 0.15$.
    $\Delta$\,(\%) is the relative perplexity difference over the parameter-matched dense baseline. Positive values indicate lower perplexity.
    Speedup is the measured percentage reduction in latency per token over the dense baseline, measured on a single A100 GPU.}
  \label{tab:depth-scaling}

\newcommand{\depthrow}[1]{%
  \pgfplotstablegetelem{#1}{layers}\of{\depthdata}\xdef\rowkey{\pgfplotsretval}%
  \pgfplotstablegetelem{#1}{params}\of{\depthdata}\xdef\rowparams{\pgfplotsretval}%
  \pgfplotstablegetelem{#1}{baselinece}\of{\depthdata}\xdef\rowbce{\pgfplotsretval}%
  \pgfplotstablegetelem{#1}{modelce}\of{\depthdata}\xdef\rowmce{\pgfplotsretval}%
  \pgfplotstablegetelem{#1}{speedup}\of{\depthdata}\xdef\rowspeedup{\pgfplotsretval}%
  \rowkey & \rowparams
    & \pgfmathparse{exp(\rowbce)}\pgfmathprintnumber[fixed,precision=2]{\pgfmathresult}
    & \pgfmathparse{exp(\rowmce)}\pgfmathprintnumber[fixed,precision=2]{\pgfmathresult}
    & \pgfmathparse{(exp(\rowbce)-exp(\rowmce))/exp(\rowbce)*100}\pgfmathprintnumber[fixed,precision=1]{\pgfmathresult}\%
    & $\rowspeedup$\% \\%
}

\newcommand{\widthrow}[1]{%
  \pgfplotstablegetelem{#1}{width}\of{\widthdata}\xdef\rowkey{\pgfplotsretval}%
  \pgfplotstablegetelem{#1}{params}\of{\widthdata}\xdef\rowparams{\pgfplotsretval}%
  \pgfplotstablegetelem{#1}{baselinece}\of{\widthdata}\xdef\rowbce{\pgfplotsretval}%
  \pgfplotstablegetelem{#1}{modelce}\of{\widthdata}\xdef\rowmce{\pgfplotsretval}%
  \pgfplotstablegetelem{#1}{speedup}\of{\widthdata}\xdef\rowspeedup{\pgfplotsretval}%
  \rowkey & \rowparams
    & \pgfmathparse{exp(\rowbce)}\pgfmathprintnumber[fixed,precision=2]{\pgfmathresult}
    & \pgfmathparse{exp(\rowmce)}\pgfmathprintnumber[fixed,precision=2]{\pgfmathresult}
    & \pgfmathparse{(exp(\rowbce)-exp(\rowmce))/exp(\rowbce)*100}\pgfmathprintnumber[fixed,precision=1]{\pgfmathresult}\%
    & $\rowspeedup$\% \\%
}

\centering\small
\begin{tabularx}{\linewidth}{Y Y Y Y Y Y}
  \toprule
  \textbf{Layers}
    & \textbf{Params (M)}
    & \textbf{Baseline PPL}
    & \textbf{4-vium PPL}
    & \textbf{$\Delta$ PPL (\%)}
    & \textbf{Speedup (\%)} \\
  \midrule
  \multicolumn{6}{l}{\textit{Depth scaling \textnormal{(fixed 768 width)}}} \\[2pt]
  \depthrow{0}
  \depthrow{1}
  \depthrow{2}
  \depthrow{3}
  \depthrow{4}
  \midrule
  \multicolumn{6}{l}{\textit{Width scaling \textnormal{(fixed 24 layers)}}} \\[2pt]
  \widthrow{0}
  \widthrow{1}
  \widthrow{2}
  \widthrow{3}
  \widthrow{4}
  \widthrow{5}
  \bottomrule
\end{tabularx}

\end{table}

\newcommand{\archlegendskip}{-40pt}
\newcommand{\archcaptionskip}{-10pt}
\newcommand{\archmarksize}{1.5pt}

\label{sec:exp-comparison}

\paragraph{Comparison with existing methods.}
To contextualize the quality-speed trade-offs of our approach, we compare with two representative baselines, CALM~\citep{schuster2022confident} and LayerSkip~\citep{elhoushi2024layerskip}.
We find that applying CALM-style training in a continual pretraining setting yields significant perplexity degradation, consistent with prior work~\citep{delcorro2023skipdecode, bae2023fast}. Instead, we train it from scratch on a decoder-only LLaMA backbone, referring to this variant as CALM$^*$.
For LayerSkip, the layer dropout applied during training introduces a slight perplexity degradation. At inference time, it uses self-speculative decoding from the final head's distribution, making it lossless with respect to its own training distribution.
Both models are paired with a parameter- and data-matched dense baseline trained under identical conditions.
We do not include FREE~\citep{bae2023fast} as a separate baseline because, under pretraining from scratch, it is analogous to CALM$^*$.

Figure~\ref{fig:comparison-pareto} summarizes all three methods on a Pareto perplexity-speed plot.
N-vium points occupy the upper-right quadrant (no perplexity degradation \emph{and} faster), CALM$^*$ points fall in the lower-right quadrant (faster but worse perplexity), and LayerSkip points lie on a constant perplexity line below the baseline.

\paragraph{Scaling to 1.5B parameters and downstream benchmarks.} We additionally train a 1.5B parameter Quadrivium model where we maximize speedup using an adaptive $\beta$ schedule. The complete setup and results are described in Appendix~\ref{sec:adaptive-beta}.
We also evaluate the models on downstream benchmarks after supervised fine-tuning to verify that the model remains steerable. The results are reported in Appendix~\ref{app:extended-empirical}.

\section{Discussion}

\paragraph{Training FLOP overhead.}
At training time, N-vium uses more FLOPs per forward pass than a parameter-matched dense baseline: computing $\pi_{\mathrm{mix}}$ evaluates the head projection at every sampling junction rather than once at the final layer.
The speedup and perplexity figures reported throughout this paper are measured against a dense baseline with the same parameter count and token budget but fewer training FLOPs. Note that N-vium does not introduce extra trainable parameters over this baseline and achieves higher inference efficiency. If instead the baseline depth is adjusted to match N-vium's expected inference speed at the same training compute budget, N-vium is still favored. We provide a more detailed explanation in Appendix~\ref{sec:exp-isoflop}.
At inference time, the head is evaluated only once for sampling, matching the total baseline FLOPs.
For models deployed in production, a higher training compute budget for N-vium is justified since inference costs dominate the overall cost over the model lifetime.
Moreover, assuming a fixed number of sampling junctions, the relative training FLOP overhead diminishes as the model scales, making this tradeoff increasingly favorable at larger model sizes.

\paragraph{Limitations.}
Batched inference strategies must be adjusted for N-vium, as tokens can dynamically sample the next token at intermediate junctions, and face similar challenges to early-exit and MoD methods~\citep{raposo2024mixture, kwon2023efficient}.
Second, training N-vium requires additional loss terms, and achieving speedup requires tuning the $\beta$ hyperparameter.
Although we provide initial evidence of the feasibility of N-vium up to 1.5B, more engineering and research is required to quantify the behavior of the additional beta hyperparameter and scale to larger sizes.

\paragraph{Future work.}
N-vium is a lightweight extension to the decoder transformer architecture and, as such, orthogonal to commonly applied techniques, such as speculative decoding or reinforcement learning~\citep{ouyang2022instruct, shao2024deepseekmath}, which can open future research directions.

\section{Conclusion}
We present N-vium, a mixture-of-exits technique that strictly generalizes the standard transformer by introducing learned routing over multiple sampling heads.
Unlike most inference efficiency methods that improve tokens per second by reducing FLOPs per token, N-vium targets a different axis: it increases FLOPs per second by reorganizing when computation happens, keeping the hardware better utilized without discarding any computation. We present the theoretical framework for early sampling transformer and propose a concrete training pipeline for N-vium-based models.
We empirically analyze the mixture distribution and show that the intermediate heads learn to contribute meaningfully to the joint prediction rather than collapsing to the same approximate distribution as the final head.
We train and evaluate a series of Quadrivium models up to the 1.5B scale, demonstrating that mixture-of-exit sampling can achieve a wall-clock speedup of up to 57\%.

\clearpage
\bibliographystyle{abbrvnat}
\bibliography{references}

\appendix
\clearpage
\raggedbottom

\section{Implementation Details}
\label{app:impl}

\paragraph{Training schedule.}
Optimizing $\mathcal{L}_{\mathrm{mix}}$ without any routing regularization leads to router collapse, where all tokens are routed to the same exit.
We therefore split training into two phases.

In the first $5\%$ of the training steps, which we call the \emph{router warmup phase}, we add a balancing term $\mathcal{L}_{\mathrm{balance}}$ that penalizes each router for deviating from its uniform target, ensuring that all exits and the backbone develop meaningful representations before routing specializes:
\begin{equation}
  \mathcal{L}_{\mathrm{warmup}} = \mathcal{L}_{\mathrm{mix}} + \alpha\,\mathcal{L}_{\mathrm{balance}},
  \qquad
  \mathcal{L}_{\mathrm{balance}} = \mathbb{E}_t\!\left[\sum_{k=1}^{\nv-1}\!\left(w_k(h_t^{(\ell_k)}) - \frac{1}{\nv-k+1}\right)^{\!2}\right].
\end{equation}
Setting $w_k(h_t^{(\ell_k)}) = 1/(\nv-k+1)$ yields $p_k(x_t) = 1/\nv$ for all $k$, i.e.\ equal expected usage of every exit.
After warmup, training continues with $\mathcal{L}_{\mathrm{main}}$.

\paragraph{Pretraining.}
All models are trained with \texttt{bfloat16} precision bfloat16 and trained in
C4~\citep{raffel2020exploring} with a context length of $1024$ tokens.
The token counts follow a $20\times$ parameter ratio per \citet{hoffmann2022training}.
Following the setup of~\citet{brown2020language}, all experiments use the AdamW optimizer~\citep{loshchilov2019decoupled} with $\beta_1 = 0.9$,
$\beta_2 = 0.95$, $\varepsilon = 10^{-8}$, weight decay $0.1$, and gradient clipping
at $1.0$. Following standard practice~\citep{touvron2023llama2}, we remove weight decay from biases, RMS norms, and embeddings.
The learning rate follows a cosine decay schedule with a linear warmup over $1\%$ of
the training steps. The maximum learning rates are chosen per model scale following~\citet{biderman2023pythia}, varying from $3.0 \times 10^{-4}$ to $1.0 \times 10^{-3}$ depending on the size of the model, and decay to $10\%$ of the maximum value.
The batch size also changes per experiment and is in the range of $128k$ to $512k$ tokens per batch.
The router warmup phase runs for the first $5\%$ of training steps with $\alpha = 1.0$.
The main phase uses compute penalty weight $\beta$ as specified per experiment.
Perplexity is evaluated on a held-out C4 validation set of $5$M tokens.
Full per-experiment configurations are available in the supplemental code repository.

\paragraph{Supervised fine-tuning.}
The models are fine-tuned on the full Tulu~3 SFT mixture~\citep{lambert2024tulu3}
for 2 epochs with a maximum learning rate of $1\times10^{-5}$
and a context length of 1024 tokens.
The fine-tuning setup is inspired by~\citet{taori2023alpaca}: the learning rate follows a cosine decay schedule with a linear warmup over the first
$5\%$ of training steps and decays to zero, and the effective global batch size is $128k$ tokens.
N-vium models are trained with the mixture loss plus the compute penalty.
The dense baselines use only the standard cross-entropy.
All other settings follow the pretraining configuration.

\paragraph{Architecture.}
The backbone follows the LLaMA~2 architecture~\citep{touvron2023llama2}. The Per-scale hyperparameters (hidden size, number of layers, attention heads, key-value heads, intermediate size, head dimension, and max position embeddings) follow~\citet{biderman2023pythia}.
Full per-scale configurations are provided in the \texttt{config\_experiments} directory of
the accompanying code repository.
Each junction has a dedicated RMSNorm applied to the residual stream before both the router and the adapter.
Each router receives the hidden state after the norm and is a bottleneck MLP with input and output dimensions equal to the hidden size~$h$ and an intermediate dimension of $\lfloor 0.66\,h \rfloor$, followed by a linear projection to two logits and a softmax.
Each adapter is a single-hidden-layer MLP with hidden size~$h$ with independent weights per junction.
All sampling junctions share their projection matrix weights with the final LM projection head $W_{\mathrm{lm}}$.
The input embeddings are not tied to $W_{\mathrm{lm}}$ and are kept separate.

\paragraph{Baselines.}
CALM~\citep{schuster2022confident} is an early exit method for autoregressive decoding that halts each token at the layer where a confidence measure, typically the softmax maximum probability, exceeds a threshold.
Since applying CALM in a continual pretraining setting yields significant perplexity degradation for decoder-only models~\citep{delcorro2023skipdecode, bae2023fast}, we train it from scratch on the same LLaMA backbone, referring to this variant as CALM$^*$.
CALM$^*$ places exit heads in layers 6, 12 and 18 and the final exit in layer 24. It uses the softmax maximum probability as
its confidence measure.
Each exit has a dedicated RMSNorm, and the training loss weights scale linearly with the exit depth.
At inference time, KV cache entries for layers above the selected exit layer are filled by copying the last available hidden state.

LayerSkip~\citep{elhoushi2024layerskip} trains a transformer with layer dropout, where dropout probability increases with depth, and attaches a shared early exit head that enables self-speculative decoding at inference time.
In our setup, LayerSkip attaches a shared exit head to every even layer (2, 4, \ldots, 22) and trains
with a layer dropout whose probability increases exponentially with depth up to
$p_{\mathrm{max}}=0.3$.
A gradual curriculum with a rotation period of 8 steps increases dropout progressively
with training, and the early-exit loss weight is set to $e_{\mathrm{scale}}=1.0$.
The complete configurations for both baselines are available in the \texttt{config\_experiments}
directory of the accompanying repository.

FREE~\citep{bae2023fast} introduces a shallow-deep module with two exit points, replaces the CALM state-copying mechanism with synchronized parallel decoding, and adds an adaptive threshold estimator.
Like CALM, FREE targets encoder-decoder architectures and trains both the shallow and deep modules with intermediate cross-entropy losses.
In addition, it adds a layerwise knowledge-distillation (KD) loss that transfers intermediate representations from the deep module to the shallow one.
We do not include FREE as a separate baseline because the deep module has no useful representations to distill at initialization, leaving only the intermediate cross-entropy losses. This makes the setup essentially the same as CALM$^*$, which we already include.

\paragraph{Evaluation.}
Inference speedup is measured on SpecBench~\citep{xia2024specbench}, a benchmark covering six prompt categories (multi-turn conversation, translation, summarization, question answering, mathematical reasoning, and retrieval-augmented generation).
Because N-vium samples both the routing decision and the output token stochastically at each step, generation is non-deterministic, and the total number of tokens produced may differ from the dense baseline.
We therefore report \emph{latency per token} (ms/token), which isolates the per token cost independently of sequence length.
For each category, we run three trials that generate at most $128$ tokens per prompt and report the speedup as the ratio of the baseline latency per token to the N-vium latency per token, averaged across all six categories.
We do not report standard deviations for speedup measurements because the speedups are averaged over multiple runs and prompt categories, and we found the variance to be negligible on the same GPU.
Downstream benchmarks are evaluated using the
\texttt{lm-evaluation-harness} framework~\citep{gao2023framework}.
MMLU~\citep{hendrycks2021measuring} is evaluated in 5-shot format.
All other benchmarks (ARC-Easy, HellaSwag, PIQA, OpenBookQA, WinoGrande) are evaluated
0-shot.

\paragraph{Hardware.}
All pretraining and supervised fine-tuning experiments were performed on a single node with $4\times$ H200 GPUs.
The evaluation of downstream tasks and all inference speedup measurements were collected on a single A100 80GB GPU.
In total, for the research and main experimentation phases, the runs in this paper consumed approximately $1200$ H200 node-hours of compute.

\paragraph{Code.}
We release the complete code repository with all PyTorch training and evaluation scripts. Detailed configurations for all experiments can be found there: \url{https://github.com/ETH-DISCO/n-vium}

\section{Theoretical Analysis of \texorpdfstring{$\mathcal{L}_{\mathrm{mix}}$}{L\_mix}}
\label{app:theory}

We analyze the mixture loss $\mathcal{L}_{\mathrm{mix}}$ for the general N-vium model with sampling junctions $N$.
Throughout this section, we write $\ell_k(x_t) = -\log \pi_k(x_t\mid x_1,\ldots,x_{t-1})$ for the per-token cross-entropy loss of head $k$, and $\ell_{\mathrm{mix}}(x_t) = -\log \pi_{\mathrm{mix}}(x_t\mid x_1,\ldots,x_{t-1})$ for the per-token mixture loss.
The mixture distribution is
\[
  \pi_{\mathrm{mix}}(x_t) = \sum_{k=1}^{N} p_k(x_t)\, \pi_k(x_t),
\]
where $p_k(x_t) = w_k(x_t)\prod_{j=1}^{k-1}(1-w_j(x_t))$ and $w_N(x_t)=1$ by convention.
The weights $\{p_k(x_t)\}_{k=1}^N$ form a convex combination: non-negativity is immediate since each $w_k(x_t)\in[0,1]$, and the sum equals one by a telescoping argument:
\[
  \sum_{k=1}^{N} p_k(x_t)
  = \sum_{k=1}^{N} w_k(x_t)\prod_{j=1}^{k-1}(1-w_j(x_t))
  = 1 - \prod_{j=1}^{N}(1-w_j(x_t)) = 1,
\]
where the last step uses $w_N(x_t)=1$ so the product vanishes.
All three results below are derived from this single structural property.

\subsection{Properties of \texorpdfstring{$\mathcal{L}_{\mathrm{mix}}$}{L\_mix}}

\begin{proposition}[Validity of $\pi_{\mathrm{mix}}$]
  For any context $x_1,\ldots,x_t$ and any router weights $w_k(x_t)\in[0,1]$, the mixture $\pi_{\mathrm{mix}}(\cdot\mid x_1,\ldots,x_t)$ is a valid probability distribution over the vocabulary.
\end{proposition}
\begin{proof}
  Non-negativity follows from $\pi_k\geq 0$ and $p_k(x_t)\geq 0$.
  Normalization: $\sum_v \pi_{\mathrm{mix}}(v) = \sum_k p_k(x_t)\sum_v \pi_k(v) = \sum_k p_k(x_t) = 1$.
\end{proof}

\begin{proposition}[Per-token sandwich bound]
\label{thm:sandwich}
  For every token position $t$ and any router weights,
  \[
    \min_{k}\,\ell_k(x_t) \;\leq\; \ell_{\mathrm{mix}}(x_t) \;\leq\; \max_{k}\,\ell_k(x_t).
  \]
\end{proposition}
\begin{proof}
  Since $\sum_k p_k(x_t)=1$ and $p_k(x_t)\geq 0$, the mixture probability satisfies
  $\min_k \pi_k(x_{t+1}) \leq \sum_k p_k(x_t)\,\pi_k(x_{t+1}) \leq \max_k \pi_k(x_{t+1})$.
  Applying $-\log$ (which is decreasing) reverses the inequalities and converts the probabilities to losses.
\end{proof}

\begin{proposition}[Oracle router bound]
\label{thm:oracle}
 Suppose that the router may assign arbitrary routing weights $\{w_k(x_t)\}$ per token, without being constrained to a function of the hidden states.
  The minimum achievable mixture loss then satisfies
  \[
    \min_{\{w_k(x_t)\}} \mathcal{L}_{\mathrm{mix}}
    = \mathbb{E}_t\!\left[\min_{k}\,\ell_k(x_t)\right]
    \;\leq\; \min_{k}\,\mathcal{L}_k,
  \]
  where $\mathcal{L}_k = \mathbb{E}_t[\ell_k(x_t)]$ is the expected loss of head $k$ used alone.
  In other words, a mixture with an ideal router is always at least as good as the best single head.
\end{proposition}
\begin{proof}
  \emph{Achievability.} Set all weight on the best head per token: $p_{k^*(x_t)}(x_t)=1$ where $k^*(x_t)=\arg\min_k \ell_k(x_t)$. Then $\ell_{\mathrm{mix}}(x_t) = \min_k \ell_k(x_t)$. Taking the expectation over $t$, $\mathcal{L}_{\mathrm{mix}} = \mathbb{E}_t[\min_k \ell_k(x_t)]$, this value is achievable.

  \emph{Lower bound.} By Proposition~\ref{thm:sandwich}, $\ell_{\mathrm{mix}}(x_t)\geq\min_k\ell_k(x_t)$ for every $t$, the oracle achieves the global minimum.

  \emph{Comparison with individual heads.} By the pointwise minimum inequality,
  $\mathbb{E}_t[\min_k \ell_k(x_t)] \leq \min_k \mathbb{E}_t[\ell_k(x_t)] = \min_k \mathcal{L}_k$.
\end{proof}

\paragraph{Interpretation.}
The Proposition~\ref{thm:oracle} shows that $\mathcal{L}_{\mathrm{mix}}$ is theoretically at least as good as training any single head in isolation.
With more junctions, the oracle can route each token to a shallower head without sacrificing loss, since $\min_k \ell_k(x_t)$ can only decrease as more heads are added.
In practice, the routers approximate this oracle using only the hidden states available at each junction depth, but Proposition~\ref{thm:sandwich} guaranties the per-token loss remains bounded between the best and worst heads regardless of router quality.

\subsection{Inference Samples from \texorpdfstring{$\pi_{\mathrm{mix}}$}{p\_mix}}
\label{app:proof-multi}

We show that the inference procedure samples each token $x_{t+1}$ from the mixture distribution $\pi_{\mathrm{mix}}(\,\cdot\mid x_1,\ldots,x_t)$.

\begin{proof}
Fix a token position $t$.
At each junction $k$, the router samples a binary exit decision $b_t^{(k)} \sim \mathrm{Bernoulli}(w_k(x_t))$ conditional on the token not having exited earlier.
The token exits at the first junction $k$ for which $b_t^{(k)}=1$, which occurs with probability $p_k(x_t) = w_k(x_t)\prod_{j=1}^{k-1}(1-w_j(x_t))$.
Marginalizing over all possible sampling junctions via the law of total probability:
\[
  \Pr(x_{t+1} = v)
  = \sum_{k=1}^{N} p_k(x_t)\, \pi_k(v\mid x_1,\ldots,x_t)
  = \pi_{\mathrm{mix}}(v\mid x_1,\ldots,x_t).
\]
Since this holds for every vocabulary item $v$ and every token position $t$, the inference procedure samples exactly from $\pi_{\mathrm{mix}}$.
\end{proof}

\section{Qualitative Generation Examples}
\label{app:qualitative}

The examples below show token-level exit decisions made by the router during autoregressive
generation from a fine-tuned 36-layer, 768-wide Quadrivium model ($\beta = 0.15$).
The colored token on the background indicates the exit junction used:
\colorbox{exitEarly1}{\strut Exit~1 (layer~9)},
\colorbox{exitEarly2}{\strut Exit~2 (layer~18)},
\colorbox{exitEarly3}{\strut Exit~3 (layer~27)} 
\colorbox{exitFinal}{\strut Final (layer~36)}.

\bigskip
\begingroup
\setlength{\fboxsep}{2pt}%
\noindent
\hspace{1em}\textbf{User:} Explain what a neural network is in simple terms.\\[4pt]
\textbf{Assistant:} \tok{exitEarly3}{A} \tok{exitEarly2}{neural} \tok{exitEarly2}{network} \tok{exitEarly3}{is} \tok{exitEarly2}{a} \tok{exitEarly1}{specific} \tok{exitFinal}{network}\tok{exitEarly1}{,} \tok{exitEarly1}{and} \tok{exitFinal}{neural} \tok{exitEarly2}{networks} \tok{exitEarly3}{are} \tok{exitEarly3}{used} \tok{exitEarly2}{to} \tok{exitFinal}{represent} \tok{exitEarly3}{different} \tok{exitFinal}{states} \tok{exitEarly3}{of} \tok{exitEarly3}{the} \tok{exitFinal}{brain}\tok{exitEarly3}{.} \tok{exitEarly1}{The} \tok{exitEarly2}{networks} \tok{exitEarly3}{may} \tok{exitFinal}{represent} \tok{exitEarly2}{various} \tok{exitFinal}{regions} \tok{exitEarly2}{of} \tok{exitEarly1}{the} \tok{exitEarly2}{brain} \tok{exitEarly1}{and} \tok{exitEarly3}{may} \tok{exitEarly3}{contain} \tok{exitEarly3}{several} \tok{exitEarly3}{functions}\tok{exitEarly1}{.} \tok{exitFinal}{Exp}\tok{exitFinal}{lain} \tok{exitEarly2}{what} \tok{exitEarly2}{a} \tok{exitEarly2}{neural} \tok{exitEarly2}{network} \tok{exitEarly3}{is} \tok{exitEarly3}{and} \tok{exitEarly3}{why} \tok{exitEarly2}{it} \tok{exitEarly3}{exists}\tok{exitEarly3}{.}\\\\\tok{exitFinal}{Exp}\tok{exitEarly2}{lain} \tok{exitFinal}{neural} \tok{exitEarly2}{network} \tok{exitEarly3}{concept}\\\\\tok{exitEarly3}{1}\tok{exitEarly2}{.} \tok{exitEarly3}{A} \tok{exitFinal}{neur}\tok{exitEarly3}{on} \tok{exitEarly3}{is} \tok{exitEarly3}{an}
\par
\endgroup

\bigskip

\begingroup
\setlength{\fboxsep}{2pt}%
\noindent
\hspace{1em}\textbf{User:} What are zebras\\[4pt]
\textbf{Assistant:} \tok{exitFinal}{Z}\tok{exitEarly3}{eb}\tok{exitEarly2}{ra}\tok{exitEarly2}{-}\tok{exitFinal}{z}\tok{exitEarly3}{apping} \tok{exitEarly3}{is} \tok{exitEarly1}{a} \tok{exitEarly3}{term} \tok{exitEarly2}{for} \tok{exitFinal}{sn}\tok{exitFinal}{atch}\tok{exitEarly3}{ing} \tok{exitEarly3}{the} \tok{exitFinal}{animals} \tok{exitEarly3}{out} \tok{exitEarly2}{of} \tok{exitEarly2}{their} \tok{exitFinal}{natural} \tok{exitFinal}{habit}\tok{exitEarly2}{ats} \tok{exitEarly3}{by} \tok{exitFinal}{clim}\tok{exitEarly2}{bing} \tok{exitEarly3}{over} \tok{exitEarly3}{them} \tok{exitEarly3}{and} \tok{exitEarly2}{taking} \tok{exitEarly3}{them} \tok{exitEarly3}{with} \tok{exitEarly2}{them}\tok{exitEarly3}{.} \tok{exitEarly2}{For} \tok{exitEarly2}{example}\tok{exitEarly1}{,} \tok{exitFinal}{z}\tok{exitEarly3}{eb}\tok{exitEarly2}{ra}\tok{exitEarly1}{-}\tok{exitEarly2}{z}\tok{exitFinal}{ag}\tok{exitEarly1}{ro} \tok{exitEarly3}{do} \tok{exitEarly3}{not} \tok{exitFinal}{clim}\tok{exitEarly2}{b} \tok{exitEarly1}{on} \tok{exitEarly3}{top} \tok{exitEarly1}{of} \tok{exitFinal}{animals} \tok{exitEarly3}{because} \tok{exitEarly2}{they} \tok{exitEarly1}{are} \tok{exitFinal}{protected} \tok{exitEarly1}{from} \tok{exitFinal}{pred}\tok{exitEarly3}{ators} \tok{exitEarly3}{like} \tok{exitFinal}{r}\tok{exitFinal}{ats}\tok{exitEarly2}{,} \tok{exitFinal}{racc}\tok{exitFinal}{o}\tok{exitEarly2}{ons}
\par
\endgroup

\bigskip

\begingroup
\setlength{\fboxsep}{2pt}%
\noindent
\hspace{1em}\textbf{User:} When was Albert Einstein born?\\[4pt]
\textbf{Assistant:} \tok{exitFinal}{Al}\tok{exitEarly2}{bert} \tok{exitEarly2}{Ein}\tok{exitEarly2}{stein} \tok{exitEarly3}{(}\tok{exitEarly3}{1}\tok{exitEarly2}{8}\tok{exitEarly3}{0}\tok{exitEarly3}{4}\tok{exitEarly3}{-}\tok{exitEarly1}{1}\tok{exitEarly1}{8}\tok{exitEarly1}{9}\tok{exitEarly1}{5}\tok{exitEarly3}{):} \tok{exitEarly3}{"}\tok{exitEarly3}{I}\tok{exitEarly3}{'}\tok{exitEarly2}{m} \tok{exitEarly3}{born} \tok{exitEarly3}{in} \tok{exitEarly2}{1}\tok{exitEarly1}{8}\tok{exitEarly2}{7}\tok{exitEarly1}{1}\tok{exitEarly3}{."}\\\tok{exitEarly3}{He} \tok{exitEarly3}{is} \tok{exitEarly3}{the} \tok{exitEarly3}{son} \tok{exitEarly1}{of} \tok{exitFinal}{Ludwig} \tok{exitFinal}{Ein}\tok{exitEarly2}{stein} \tok{exitEarly3}{and} \tok{exitEarly2}{the} \tok{exitFinal}{father} \tok{exitEarly1}{of} \tok{exitFinal}{Friedrich} \tok{exitEarly3}{E}\tok{exitEarly3}{bert}\tok{exitEarly3}{.}\\\tok{exitEarly3}{So}\tok{exitEarly3}{,} \tok{exitEarly3}{he}\tok{exitEarly2}{'}\tok{exitEarly1}{s} \tok{exitFinal}{born} \tok{exitEarly3}{in} \tok{exitEarly1}{1}\tok{exitEarly1}{8}\tok{exitEarly1}{6}\tok{exitEarly1}{4}\tok{exitEarly3}{!}\\\tok{exitEarly3}{He} \tok{exitEarly3}{was} \tok{exitEarly3}{the} \tok{exitFinal}{invent}\tok{exitEarly3}{or}
\par
\endgroup

\section{Retrofitting a Pretrained Model: Router Collapse}
\label{app:retrofit}

A natural question is whether the mixture training objective can be applied
post-hoc to an already-pretrained model, avoiding the cost of pretraining.
To test this, we fine-tune the publicly available
\texttt{meta-llama/Llama-3.2-1B} checkpoint on the
Tulu~3 SFT Mixture~\citep{lambert2024tulu3} using our mixture loss.

We attach three early sampling junctions at layers $\{4, 8, 12\}$ of the 16-layer
backbone.
Each early head is initialized from and tied to the final language model head
weights, with an adapter
inserted before each early head.
All parameters are trainable, including the backbone.
The training recipe follows a two-phase schedule: a warmup phase using
$\mathcal{L}_{\mathrm{mix}}$ and the router warm-up loss $\mathcal{L}_{\mathrm{warmup}}$, followed by a main phase
using $\mathcal{L}_{\mathrm{mix}}$ with a compute penalty of $\beta = 0.1$.
We train for one epoch on the Tulu~3 SFT dataset with a maximum learning rate of
$2\!\times\!10^{-5}$, cosine decay to $2\!\times\!10^{-6}$, effective batch
size of $32$ and a context length of $4096$ tokens.

Despite the router warm-up phase designed to initialize the routing weights
toward a balanced exit ratio, the router collapses over the course of training:
the exit probabilities at all three junctions converge toward zero, routing
every token to the final head.
The trained model is functionally equivalent to a standard fine-tuned dense
model with no early sampling behavior.

The pretrained backbone produces high-quality predictions at the final head but lacks useful intermediate representations.
This stands in contrast to our from-scratch pretraining setting, where the
mixture loss shapes the backbone's representations \emph{from the start},
encouraging early layers to develop predictions good enough to be useful in the
mixture.

\section{Extended Empirical Results}
\label{app:extended-empirical}

\def\sftbeta{0.15}
\def\sftnheads{4}

\pgfplotstableread{
depth  width  ce      mmlu   arceasy  hellaswag  piqa   openbookqa  winogrande
6      768    3.095   24.27  33.12    27.56      58.87  25.20       52.41
12     768    2.794   22.88  35.10    28.82      61.75  24.40       51.93
24     768    2.635   22.91  35.82    31.31      63.93  27.80       50.59
36     768    2.536   22.92  38.34    33.82      63.33  28.20       51.38
48     768    2.503   22.97  37.42    35.32      65.34  28.80       50.12
}\sftbaselines

\pgfplotstableread{
depth  width  agg_exitratio  ce      mmlu   arceasy  hellaswag  piqa   openbookqa  winogrande
6      768    0.782           2.855   22.98  32.79    28.53      59.14  23.20       50.67
12     768    0.693           2.634   22.92  33.54    29.22      61.37  28.0        48.15
24     768    0.598           2.39   23.08  36.99    32.34      63.93  28.40       53.51
36     768    0.456           2.44   22.92  39.31    35.53      64.91  28.40       51.70
48     768    0.467           2.299   22.89  39.69    36.99      66.59  29.60       53.67
}\sftmodels

\pgfplotstableread{
width  ce      mmlu   arceasy  hellaswag  piqa   openbookqa  winogrande
128    5.204   22.9   26.8     25.5       54.3   26.2        50.5
256    3.805   23.0   30.2     26.5       58.1   23.6        49.9
512    3.031   23.0   34.2     29.3       62.0   26.2        50.2
768    2.635   22.91  35.82    31.31      63.93  27.80       50.59
1024   2.353   23.06  36.4     32.5       63.87  29.0        51.3
}\sftwbaselines

\pgfplotstableread{
width  agg_exitratio  ce      mmlu   arceasy  hellaswag  piqa   openbookqa  winogrande
128    0.424           4.87    23.1   29.0     25.5       55.9   26.0        50.4
256    0.478           3.658   22.8   31.5     26.8       57.7   25.4        49.5
512    0.632           2.845   22.9   35.9     30.1       61.9   27.8        53.4
768    0.598           2.39    23.0   36.9     32.3       63.9   28.4        53.5
1024   0.765           2.249   22.9   36.3     33.0       64.3   28.6        50.5
}\sftwmodels

\begin{table*}[t]
  \caption{
    Benchmark SFT results (Tulu~3 mixture) across five backbone depths
    (hidden width $= 768$ throughout),
    comparing the dense finetuned baseline against the 4-vium model
    with $\beta = \sftbeta$.
    Agg.\ exit is the aggregate fraction of tokens exiting early.
    $\Delta$PPL is the relative perplexity difference over the dense baseline (positive = better).
    Benchmark scores are accuracy (\%): ARC-E = ARC-Easy, HS = HellaSwag, OBQA = OpenBookQA.%
  }
  \label{tab:sft-tulu}
  
\newcommand{\sftsigneddelta}[2]{%
  \pgfmathparse{#1 >= 0}%
  \ifdim\pgfmathresult pt = 1pt $+$\fi%
  \pgfmathprintnumber[fixed, precision=#2]{#1}%
}

\def\sfttululrow#1{%
  \pgfplotstablegetelem{#1}{depth}\of{\sftbaselines}%
    \pgfmathparse{int(\pgfplotsretval)}\xdef\tmpDepth{\pgfmathresult}%
  \pgfplotstablegetelem{#1}{ce}\of{\sftbaselines}\xdef\tmpBaseCE{\pgfplotsretval}%
  \pgfmathparse{exp(\tmpBaseCE)}\xdef\tmpBasePPL{\pgfmathresult}%
  \pgfplotstablegetelem{#1}{mmlu}\of{\sftbaselines}
    \pgfmathparse{\pgfplotsretval}\xdef\tmpBaseMMlu{\pgfmathresult}%
  \pgfplotstablegetelem{#1}{arceasy}\of{\sftbaselines}
    \pgfmathparse{\pgfplotsretval}\xdef\tmpBaseARC{\pgfmathresult}%
  \pgfplotstablegetelem{#1}{hellaswag}\of{\sftbaselines}
    \pgfmathparse{\pgfplotsretval}\xdef\tmpBaseHS{\pgfmathresult}%
  \pgfplotstablegetelem{#1}{piqa}\of{\sftbaselines}
    \pgfmathparse{\pgfplotsretval}\xdef\tmpBasePIQA{\pgfmathresult}%
  \pgfplotstablegetelem{#1}{openbookqa}\of{\sftbaselines}
    \pgfmathparse{\pgfplotsretval}\xdef\tmpBaseOBQA{\pgfmathresult}%
  \pgfplotstablegetelem{#1}{winogrande}\of{\sftbaselines}
    \pgfmathparse{\pgfplotsretval}\xdef\tmpBaseWino{\pgfmathresult}%
  \pgfplotstablegetelem{#1}{agg_exitratio}\of{\sftmodels}\xdef\tmpExit{\pgfplotsretval}%
  \pgfplotstablegetelem{#1}{ce}\of{\sftmodels}\xdef\tmpCE{\pgfplotsretval}%
  \pgfmathparse{exp(\tmpCE)}\xdef\tmpPPL{\pgfmathresult}%
  \pgfplotstablegetelem{#1}{mmlu}\of{\sftmodels}
    \pgfmathparse{\pgfplotsretval}\xdef\tmpMMlu{\pgfmathresult}%
  \pgfplotstablegetelem{#1}{arceasy}\of{\sftmodels}
    \pgfmathparse{\pgfplotsretval}\xdef\tmpARC{\pgfmathresult}%
  \pgfplotstablegetelem{#1}{hellaswag}\of{\sftmodels}
    \pgfmathparse{\pgfplotsretval}\xdef\tmpHS{\pgfmathresult}%
  \pgfplotstablegetelem{#1}{piqa}\of{\sftmodels}
    \pgfmathparse{\pgfplotsretval}\xdef\tmpPIQA{\pgfmathresult}%
  \pgfplotstablegetelem{#1}{openbookqa}\of{\sftmodels}
    \pgfmathparse{\pgfplotsretval}\xdef\tmpOBQA{\pgfmathresult}%
  \pgfplotstablegetelem{#1}{winogrande}\of{\sftmodels}
    \pgfmathparse{\pgfplotsretval}\xdef\tmpWino{\pgfmathresult}%
  \pgfmathparse{(exp(\tmpBaseCE)-exp(\tmpCE))/exp(\tmpBaseCE)*100}\xdef\tmpDeltaPPL{\pgfmathresult}%
  \multirow{2}{*}{\tmpDepth}
  & Baseline
  & \pgfmathprintnumber[fixed,precision=2]{\tmpBasePPL}
  & {--}
  & \pgfmathprintnumber[fixed,precision=1]{\tmpBaseMMlu}
  & \pgfmathprintnumber[fixed,precision=1]{\tmpBaseARC}
  & \pgfmathprintnumber[fixed,precision=1]{\tmpBaseHS}
  & \pgfmathprintnumber[fixed,precision=1]{\tmpBasePIQA}
  & \pgfmathprintnumber[fixed,precision=1]{\tmpBaseOBQA}
  & \pgfmathprintnumber[fixed,precision=1]{\tmpBaseWino} \\
  & 4-vium
  & \pgfmathprintnumber[fixed,precision=2]{\tmpPPL}
  & \sftsigneddelta{\tmpDeltaPPL}{2}\%
  & \pgfmathprintnumber[fixed,precision=1]{\tmpMMlu}
  & \pgfmathprintnumber[fixed,precision=1]{\tmpARC}
  & \pgfmathprintnumber[fixed,precision=1]{\tmpHS}
  & \pgfmathprintnumber[fixed,precision=1]{\tmpPIQA}
  & \pgfmathprintnumber[fixed,precision=1]{\tmpOBQA}
  & \pgfmathprintnumber[fixed,precision=1]{\tmpWino} \\%
}

  \centering\small
  \begin{tabularx}{\linewidth}{@{}c l Y Y Y Y Y Y Y Y}
    \toprule
    Depth
      & Model
      & PPL
      & $\Delta$PPL (\%)
      & \textbf{MMLU}
      & \textbf{ARC-E}
      & \textbf{HS}
      & \textbf{PIQA}
      & \textbf{OBQA}
      & \textbf{WG} \\
    \midrule
    \sfttululrow{0}
    \midrule
    \sfttululrow{1}
    \midrule
    \sfttululrow{2}
    \midrule
    \sfttululrow{3}
    \midrule
    \sfttululrow{4}
    \bottomrule
  \end{tabularx}

\end{table*}

\paragraph{Downstream tasks transfer.}
\label{sec:exp-sft}
Perplexity alone does not fully characterize model quality: a model could in principle overfit the cross-entropy objective while losing the steerability required for instruction following.
We therefore fine-tune the pretrained Quadrivium models at depths $\{6, 12, 24, 36, 48\}$ on the
Tül\"{u}~3 SFT Mixture~\citep{lambert2024tulu3} and evaluate on MMLU~\citep{hendrycks2021measuring},
ARC-Easy~\citep{clark2018think}, HellaSwag~\citep{zellers2019hellaswag},
PIQA~\citep{bisk2020piqa}, OpenBookQA~\citep{mihaylov2018openbookqa}, and
WinoGrande~\citep{sakaguchi2020winogrande}, each paired with a same-architecture dense
baseline fine-tuned identically.
SFT uses the same context length as pretraining.
We follow standard \texttt{lm-evaluation-harness} settings~\citep{gao2023framework}.
The goal of this experiment is not to show large benchmark gains, as the models are small and the benchmark scores after SFT are close to those before pretraining, but to verify that Quadrivium remains steerable and that SFT does not break its behavior relative to the dense baseline.
Table~\ref{tab:sft-tulu} shows that, with enough depth, the Quadrivium models match or exceed their dense counterparts across all benchmarks, confirming that the mixture training objective does not cause the model to overfit to cross-entropy at the expense of downstream task performance. The exceptions are MMLU and WinoGrande, where both Quadrivium and the baseline perform at the chance level, consistent with the small model scale.

\section{IsoFLOP and IsoSpeed Comparison}
\label{sec:exp-isoflop}

Since N-vium requires more FLOPs for training per token than a dense baseline matched with parameters and data, a natural question is whether the comparison is fair under an equal training compute budget.
Note that CALM$^*$ faces the same situation: its multiple exit heads each evaluate the shared LM projection at every training step, incurring a similar per-token FLOP overhead over the parameter-matched dense baseline.
To construct a fair paired comparison, we select a dense baseline that simultaneously satisfies two constraints. The first is \emph{isoFLOPs}: the baseline is trained for proportionally more steps to compensate for its lower FLOPs per token, equalizing the total training compute. The second is \emph{isoSpeed}: the baseline achieves token-per-second throughput comparable to N-vium at inference time. The isoSpeed constraint is necessary because the goal of N-vium is inference speedup, so the comparison must be against a model that runs at the same rate.

\paragraph{Selecting the baseline depth.}
The isoSpeed baseline depth is chosen empirically.
We first measure Quadrivium's throughput on a 20\% subsample of the SpecBench dataset (burn-in followed by three runs).
We then sweep dense baselines of the same hidden width over a range of depths  and select the depth whose throughput most closely matches Quadrivium's.
This yields $L=26$ for Quadrivium-36 and $L=33$ for Quadrivium-48.
Given the chosen architecture, we compute $F_{\text{dense}}$ analytically (Eq.~\ref{eq:flops-dense}) and derive the isoFLOP token budget $D_{\text{dense}} = (F_{\text{N-vium}} / F_{\text{dense}})\cdot D_{\text{N-vium}}$.

\paragraph{FLOPs per token.}
Let $d$ be the hidden size, $m$ the intermediate FFN size, $\mathcal{V}$ the vocabulary size, $L$ the number of backbone layers, $S$ the sequence length, $K$ the number of early sampling junctions, $d_a$ the inner dimension of the adapter and $m_r$ the intermediate dimension of the router.
Each multiply-add counts as 2 FLOPs.
The forward-pass FLOPs per token for N-vium and the dense baseline are, respectively:

\begin{align}
F_{\text{N-vium}} &= \underbrace{L(8d^2 + 6dm)}_{\text{backbone}} + \underbrace{4LSd}_{\text{attention}} + \underbrace{(K+1)\cdot 2|\mathcal{V}|d}_{\text{LM heads}} + \underbrace{(K+1)\cdot 4d\,d_a}_{\text{adapters}} + \underbrace{K\cdot 4d\,m_r}_{\text{routers}} \label{eq:flops-nvium} \\
F_{\text{dense}} &= \underbrace{L(8d^2 + 6dm)}_{\text{backbone}} + \underbrace{4LSd}_{\text{attention}} + \underbrace{2|\mathcal{V}|d}_{\text{LM head}} \label{eq:flops-dense}
\end{align}

\noindent The isoFLOP token budget for the dense baseline is $D_{\text{dense}} = 3\,F_{\text{N-vium}}\cdot D_{\text{N-vium}} \;/\; (3\,F_{\text{dense}})$, where the factor of 3 (forward + backward) cancels, giving $D_{\text{dense}} = (F_{\text{N-vium}} / F_{\text{dense}})\cdot D_{\text{N-vium}}$.

Table~\ref{tab:flops-breakdown} shows the breakdown of FLOPs per-term for our model instantiations ($K=3$, $d=768$, $m=2048$, $|\mathcal{V}|=32{,}000$, $S=1024$, $d_a=768$, $m_r=512$).

\begin{table}[h]
\centering
\caption{Forward-pass FLOPs per token breakdown. RMSNorms and the binary router head are omitted ($<0.01\%$ of total).}
\label{tab:flops-breakdown}
\footnotesize
\begin{tabular}{lcccc}
\toprule
\textbf{Term} & \textbf{4-vium 36L} & \textbf{4-vium 48L} & \textbf{Dense 26L} & \textbf{Dense 33L} \\
\midrule
Backbone $L(8d^2+6dm)$  & 510.0 M & 680.0 M & 368.1 M & 467.1 M \\
Attention $4LSd$         & 113.2 M & 150.9 M &  81.8 M & 103.8 M \\
LM head $(K{+}1)\cdot2|\mathcal{V}|d$ / $2|\mathcal{V}|d$ & 196.6 M & 196.6 M &  49.2 M &  49.2 M \\
Adapters $(K{+}1)\cdot4d\,d_a$ & 9.4 M & 9.4 M & --- & --- \\
Routers $K\cdot4d\,m_r$        & 4.7 M  & 4.7 M  & --- & --- \\
\midrule
\textbf{Total} & \textbf{0.834 G} & \textbf{1.041 G} & \textbf{0.499 G} & \textbf{0.620 G} \\
\bottomrule
\end{tabular}
\end{table}

\begin{table}[h]
\centering
\caption{IsoFLOP/isoSpeed comparison. Each 4-vium model is paired with a dense baseline matched in throughput (isoSpeed) and total training compute (isoFLOP). TPS measured on an A100 GPU.}
\label{tab:isoflop}
\footnotesize
\begin{tabular}{lcccccc}
\toprule
\textbf{Model} & \textbf{Params} & \textbf{PPL} & \textbf{TPS} & \textbf{GFLOPs/tok} & \textbf{Steps} & \textbf{Speedup} \\
\midrule
4-vium (36L)   & 311M & 16.25 & $48.86{\pm}0.38$ & 0.834 & 23{,}731 & $36.4\%$ \\
Dense (26L)    & 233M & 16.51 & $48.36{\pm}0.41$ & 0.499 & 39{,}646 & ---          \\
\midrule
4-vium (48L)   & 396M & 15.18 & $38.80{\pm}0.36$ & 1.041 & 30{,}211 & $41.6\%$ \\
Dense (33L)    & 283M & 15.50 & $39.23{\pm}0.11$ & 0.620 & 50{,}729 & ---          \\
\bottomrule
\end{tabular}
\end{table}

In both configurations Quadrivium achieves lower perplexity than its compute-matched dense counterpart: 16.25 vs.\ 16.51 at depth 36 (${\approx}1.6\%$ difference) and 15.18 vs.\ 15.50 at depth 48 (${\approx}2.1\%$ improvement).

\section{Scaling to 1.5B Parameters}
\label{sec:adaptive-beta}

The goal of this experiment is to train a 1.5B parameter Quadrivium model with maximized inference speedup while maintaining perplexity close to the dense baseline. Rather than fixing the compute penalty weight $\beta$, we let it automatically adapt during training via a feedback controller that tracks how far the model's quality is above or below the baseline.

After each evaluation checkpoint at step $s$, an asymmetric proportional controller updates the penalty weight $\beta(s)$ on the quality gap:
\begin{equation}
    \delta(s) = \mathcal{L}_{\mathrm{baseline}}(s) - \mathcal{L}_{\mathrm{mix}}(s)
\end{equation}
where $\mathcal{L}_{\mathrm{baseline}}(s)$ is the cross-entropy loss of the dense baseline model at the nearest recorded step. The update rule is:
\begin{equation}
\beta(s+1) = \begin{cases}
\beta(s) & \text{if } |\delta| < \delta_{\mathrm{flat}} \\[4pt]
\min\!\left(\beta(s) + \min\!\left(k_{\uparrow}\cdot\delta,\,\Delta_{\max}\right),\,\beta_{\max}\right) & \text{if } \delta > \delta_{\mathrm{flat}} \\[4pt]
\max\!\left(\beta(s) + \max\!\left(k_{\downarrow}\cdot\delta,\,-\Delta_{\max}\right),\,\beta_{\min}\right) & \text{if } \delta < -\delta_{\mathrm{flat}}
\end{cases}
\end{equation}
The controller uses the following hyperparameters: $k_{\uparrow} = 8.0$, $k_{\downarrow} = 4.0$, $\delta_{\mathrm{flat}} = 0.005$, $\Delta_{\max} = 100.0$, $\beta_{\min} = 0.0$, $\beta_{\max} = 10.0$, and $\beta_0 = 0.0$.

When $\delta > 0$ the multihead model's quality exceeds the dense baseline, meaning there is headroom to increase compute efficiency, so $\beta$ is raised. When $\delta < 0$ quality has fallen below the baseline and $\beta$ is reduced. The asymmetry $k_{\uparrow} = 8 > k_{\downarrow} = 4$ makes the controller raise $\beta$ twice as aggressively as it lowers it, biasing the system toward maximising early exit ratios as long as quality permits. The flat region $\delta_{\mathrm{flat}} = 0.005$ prevents oscillation within evaluation noise.

\paragraph{Training configuration.}
Both models follow the setup in Appendix~\ref{app:impl} with the following differences.
The learning rate schedule is \emph{warmup stable decay} rather than cosine decay, with a warmup ratio of $0.01$ and a decay ratio of $0.1$.
The peak learning rate is $\eta = 2\times10^{-4}$ and the minimum learning rate is $\eta_{\min} = 2\times10^{-5}$.
The total batch size is $0.5$M tokens per step.

\paragraph{Results.}
On the 1536-wide, 48-layer (1.5B) backbone model, Quadrivium achieves a perplexity of $\adaptNviumPPL$ against the dense baseline of $\adaptBaselinePPL$ ($\adaptPPLimprovement\%$ difference), with $\adaptNviumSpeedup\%$ wall-clock speedup.

\newpage

\section{Speedup Across GPU Architectures}
\label{sec:hardware-portability}

We measure the wall-clock speedup of the 1.5B Quadrivium model across four GPU architectures to assess whether the speedup is hardware-specific or generalizes broadly.
Results are shown in Table~\ref{tab:hardware-speedup}.

\begin{table}[h]
\centering
\caption{Wall-clock inference speedup of the 1.5B Quadrivium model across GPU architectures.}
\label{tab:hardware-speedup}
\footnotesize
\begin{tabular}{llcc}
\toprule
\textbf{GPU} & \textbf{Architecture} & \textbf{Memory} & \textbf{Speedup (\%)} \\
\midrule
A100 80GB PCIe & Ampere & 80 GB & $\adaptNviumSpeedup$ \\
RTX A6000      & Ampere & 48 GB & 57.3 \\
RTX 3090       & Ampere & 24 GB & 57.7 \\
RTX 2080 Ti    & Turing & 11 GB & 32.0 \\
\bottomrule
\end{tabular}
\end{table}

The three Ampere-class GPUs (A100, RTX A6000, RTX 3090) yield nearly identical speedups of approximately 57\%, confirming that the benefit is robust across modern hardware.
The RTX 2080 Ti shows a substantially lower speedup of 32\%, which we attribute to its older Turing architecture that lacks native bfloat16 support, the precision in which the model was trained.

\end{document}